\documentclass[10pt,letterpaper,oneside]{amsart}
\usepackage[square,sort,comma,numbers]{natbib}
\usepackage{times}
\usepackage{hyperref}
\usepackage{cleveref}
\usepackage{graphicx}
\usepackage[letterpaper]{geometry}
\usepackage{subcaption}

\usepackage{fullpage} 
\usepackage{bbm} 
\usepackage{amssymb} 
\usepackage{graphicx} 
\usepackage{enumerate} 
\usepackage{mathabx} 
\usepackage{xcolor} 
\usepackage{soul} 
\usepackage{tikz-cd} 
\usepackage{algorithm}
\usepackage[noend]{algpseudocode}

\let \P \relax

\newcommand{\P}{\mathbb{P}}

\newcommand{\R}{\mathbb{R}}
\newcommand{\C}{\mathbb{C}}

\newcommand{\LL}{\mathcal{L}}

\newcommand{\RR}{\mathcal{R}}

\newcommand{\PP}{\mathcal{P}}

\newcommand{\la}{\lambda}

\newcommand{\si}{\sigma}

\renewcommand{\phi}{\varphi}

\let \< \relax
\let \> \relax
\newcommand{\<}{\langle}
\newcommand{\>}{\rangle}

\newcommand{\mn}[1]{{\left\vert\kern-0.25ex\left\vert\kern-0.25ex\left\vert #1 \right\vert\kern-0.25ex\right\vert\kern-0.25ex\right\vert}}

\newcommand{\BM}{\begin{matrix}}
\newcommand{\EM}{\end{matrix}}

\DeclareMathOperator{\GL}{GL}

\DeclareMathOperator{\Span}{Span}

\DeclareMathOperator{\rank}{rank}

\DeclareMathOperator{\adj}{adj}

\DeclareMathOperator{\cone}{cone}

\DeclareMathOperator{\cof}{cof}

\DeclareMathOperator{\sign}{sgn}

\newtheorem{theorem}{Theorem}[section]
\newtheorem{definition}[theorem]{Definition}

\newtheorem{lemma}[theorem]{Lemma}
\newtheorem{example}[theorem]{Example}
\newtheorem*{example*}{Example}
\newtheorem{corollary}[theorem]{Corollary}

\newtheorem{remark}[theorem]{Remarks}

\newcommand{\vu}{{\bf u}}
\newcommand{\vv}{{\bf v}}
\newcommand{\vw}{{\bf w}}
\newcommand{\vx}{{\bf x}}
\newcommand{\va}{{\bf a}}
\newcommand{\vb}{{\bf b}}
\newcommand{\vt}{{\bf t}}
\newcommand{\ve}{{\bf e}}

\newcommand{\vnull}{{\bf 0}}

\newcommand{\cP}{\mathcal{P}}
\newcommand{\cL}{\mathcal{L}}

\renewcommand{\bf}{\mathbf}

\title{Existence of Two View Chiral Reconstructions}
\author{Andrew Pryhuber}
\address{Department of Mathematics, 
University of Washington, Box 354350, Seattle, WA 98195, USA.}
\email{andrewpryhuber@gmail.com}
\author{Rainer Sinn}
\address{Mathematisches Institut, Universit\"at Leipzig, PF 10 09 20, 04009 Leipzig, Germany.}
\email{rainer.sinn@uni-leipzig.de}
\author{Rekha R. Thomas}
\address{Department of Mathematics, University of Washington, Box
  354350, Seattle, WA 98195, USA.}
\email{rrthomas@uw.edu}
\thanks{Pryhuber and Thomas were partially supported by the NSF grant DMS-1719538}
\date{\today}

\begin{document}

\begin{abstract} 

A fundamental question in computer vision is whether a set of point pairs is the image of a scene that lies in front of two cameras. Such a scene and the cameras together are known as a {\em chiral reconstruction} of the point pairs. In this paper we provide a complete classification of $k$ point pairs for which a chiral reconstruction exists. The existence of chiral reconstructions is equivalent to the
non-emptiness of certain semialgebraic sets. We describe these sets and develop tools to certify their non-emptiness. For up to three point pairs, we prove that a chiral reconstruction always exists while
the set of five or more point pairs that do not have a chiral reconstruction is Zariski-dense. We show that for five generic point pairs, the chiral region is bounded by line segments in a Schl\"afli double six on a cubic surface with $27$ real lines. Four point pairs have a chiral reconstruction unless they belong to two non-generic combinatorial types, in which case they may or may not.  

\end{abstract}

\maketitle

\section{Introduction}
\label{sec:introduction}

A fundamental question in computer vision is whether a set of point pairs $\mathcal{P} = \{(\bf{u}_i, \bf{v}_i) \,:\, i=1,\ldots,k \}$ is the image of a set of world points $\bf{q}_i$ that are visible in two cameras. If we ignore the constraint (as is commonly done) that the points $\mathbf{q}_i$ need to lie in front of the cameras, we get a {\em projective reconstruction}~\cite{HartleyZisserman2004}. In reality though, cameras can only see points in front of them. A reconstruction that obeys this additional constraint is 
known as a {\em chiral reconstruction}~\cite{APST2020multiview,hartley1998chirality}. 
The aim of this paper is to give a complete answer to the question: Given a set of 
point pairs $
\mathcal{P} = \{ (\bf{u}_i, \bf{v}_i), i = 1,\ldots, k\}$, when does $\mathcal{P}$ admit a chiral reconstruction?

Under the assumption that the points in each image are distinct, we prove the following facts.
\begin{enumerate}
    \item A set of at most three point pairs always has a chiral reconstruction.
    \item A set of four point pairs has a chiral reconstruction unless the 
    configurations are of two specific non-generic types, in which case a chiral reconstruction may not exist.
    \item Five or more point pairs can fail to have a chiral reconstruction with positive probability (in particular, even if they are in general position).
    \item For five sufficiently generic point pairs, the problem translates to finding points  in semialgebraic regions on a real cubic surface whose boundaries are segments of lines in a Schl\"afli double six of real lines, creating an unexpected bridge 
    to classical results in algebraic geometry.
\end{enumerate}

The study of chirality was initiated in \cite{hartley1998chirality} with follow up work by Werner, Pajdla and others~\cite{werner2003combinatorial, werner2003constraint,WernerPajdla2001, werner2001oriented}.
There is no agreement on the name for reconstructions that are chiral. Hartley in \cite{hartley1998chirality} and Werner \& Pajdla in~\cite{werner2001oriented} call them {\em strong realizations}. Werner in \cite{werner2001oriented, werner2003constraint,werner2003combinatorial} calls them {\em oriented projective reconstructions}.  Here and in our previous work~\cite{APST2020multiview}, we prefer the term {\em chiral reconstruction}. 
In \cite{hartley1998chirality}, Hartley shows that chiral reconstructions 
are a special class of \emph{quasi-affine reconstructions}. See \cite[Section 5]{APST2020multiview} for a detailed account of quasi-affine reconstructions in the context of our approach to 
chirality. Hartley's work was done using projective geometry. In later work, such as 
\cite{werner2003combinatorial}, \cite{werner2003constraint} and \cite{WernerPajdla2001}, the authors use {\em oriented projective geometry} \cite{laveau1996oriented}, which also explains their use of the term {\em oriented projected reconstructions}. Following Hartley, our work uses projective geometry. Justifications  of our choice of framework can be found in \cite[Section~1]{APST2020multiview}.
We now comment briefly on how our results relate to the existing literature in computer vision and give more detailed citations in later sections.

A specific example of five point pairs in general linear position that do not admit a chiral reconstruction was given in \cite{werner2003constraint} and appears 
again in \cite{werner2003combinatorial}. Our result (3) shows that, in fact, the 
set of five point pairs without a chiral reconstruction is Zariski dense. 
The case of four point pairs covered in result (2) is the most involved since it does not assume 
genericity, and covers all possible configurations of four point pairs. We show that chiral 
reconstructions can fail in this case only in degenerate situations. To the best of our knowledge, both results (2) and (3) are novel.
Result (1) says that three point pairs admit a chiral reconstruction 
unconditionally. While this is not too hard to prove, it also does not appear in the literature and completes the story.

Result (4) establishes a new connection between chiral reconstructions of five generic point pairs and the classical 
theory of cubic surfaces from algebraic geometry. The cubic surface in $\mathbb{P}^3$ 
arises naturally in our set up 
and the two camera planes, modeled 
as $\mathbb{P}^2 \times \mathbb{P}^2$, can be obtained by {\em blowing down} this surface. The papers \cite{werner2003constraint} and \cite{werner2003combinatorial}  study chirality in the setting of $\mathbb{P}^2 \times \mathbb{P}^2$ 
using tools from oriented matroids. It was shown in \cite{werner2003constraint} that the 
{\em chiral regions} are bounded by certain conics. Our results show that these conics are obtained by blowing down a {\em Schl\"afli double six} of real lines on the cubic surface, allowing the chiral regions to be seen as semialgebraic subsets of the cubic surface bounded by lines. 
Both frameworks produce a sixth pair of points $(\mathbf{u}_0, \mathbf{v}_0)$ whose existence was derived in \cite{werner2003constraint} using the  conics in $\mathbb{P}^2 \times \mathbb{P}^2$. Our results show that the cubic surface is the 
{\em blow up} of the six points $\{\mathbf{u}_i\}$ or $\{\mathbf{v}_i\}$ in each $\mathbb{P}^2$. 
The blow up/blow down maps create a new interpretation of the results in \cite{werner2003constraint} while also offering a unified picture of chirality that transfers seamlessly between 
the space of camera {\em epipoles} and the space of {\em fundamental matrices} of $\mathcal{P}$.
We draw from, and build on, methods from the above mentioned papers from computer vision, and our own paper \cite{APST2020multiview}. Our main tools are centered in complex and real algebraic geometry 
as well as semialgebraic geometry.

\medskip
This paper is organized as follows. 
Formal definitions of projective and chiral reconstructions can be found in \Cref{sec:reconstructions}. In \Cref{sec:chiraltools}, we introduce the inequalities imposed by chirality and develop tools to certify them. 
Along the way we prove that a set of at most three point pairs always has a chiral reconstruction. 

In \Cref{sec:fourpoints}, we prove that a set of four point pairs has a chiral reconstruction when the point configurations in each view have sufficiently similar geometry, and in particular, when they are in general position. The bad cases fall into two non-generic combinatorial types. In particular, the probability of choosing four point pairs that fail to have a chiral reconstruction is 
zero.

In \Cref{sec:fivepoints} we show that when $k > 4$, point pairs that are in general linear 
position may not have a chiral reconstruction. Specific examples of this type when $k=5$ were known to Werner \cite{werner2003constraint} and as mentioned before, there are close connections between our work and that of Werner's \cite{werner2003combinatorial,werner2003constraint,WernerPajdla2001,werner2001oriented}. We make two new contributions for the case $k=5$. In \Cref{sec:fivepoints} we show that one can decide the existence of a chiral reconstruction when $k=5$ and $\mathcal{P}$ is generic, by checking $20$ discrete points. We use this test to show that the set of five point pairs that do not admit a chiral reconstruction is Zariski-dense. In other words, five point pairs do not have a chiral reconstruction with positive probability. A set of six or more point pairs can have a chiral reconstruction only if any subset of five point pairs among them have one. Hence, for any value of $k > 5$, there will be point configurations without a chiral reconstruction. Our second contribution in \Cref{sec:classicalAG} is to show that the case of $k=5$ is intimately related to the theory of cubic surfaces from classical algebraic geometry. Indeed, $\mathcal{P}$ creates a Schl\"afli double six of $12$ lines on a cubic surface, all of whose $27$ lines are real. These lines determine the boundary of the semialgebraic regions corresponding to chiral reconstructions.

\medskip
\textbf{Acknowledgements.} The question addressed in this paper 
is a natural follow up to our previous work on chirality in \cite{APST2020multiview}. We thank Sameer Agarwal for many helpful conversations. We also thank Tom\'a\v{s} Pajdla for 
pointing us to several chirality papers from the literature.

\section{Background and Notation}
We now introduce some background and notation that will be needed in the paper. 
Let $\P^n$ and $\P_\R^n$ denote $n$-dimensional projective space over complex and real numbers respectively. 
More generally, we write $\P(V)$ for the projective space over a vector space $V$, which is the set of lines in $V$.
We write $\bf{a} \sim \bf{b}$ if $\bf{a}$ and $\bf{b}$ are the same points in projective space, and reserve $\bf{a}=\bf{b}$ to mean coordinate-wise equality.
A {\em projective camera} is a matrix 
in $\P(\R^{3 \times 4})$ of rank three,
i.e., a $3\times 4$ real matrix defined only up to scaling, hence naturally a point in the projective space over the vector space $\R^{3\times 4}$.
Usually, an affine representative of a projective camera is fixed and we block-partition such a matrix as 
$A = \begin{bmatrix} G & \bf{t} \end{bmatrix} \in \R^{3 \times 4}$ where 
$G \in \R^{3 \times 3}$ and $\bf{t} \in \R^3$. 
The {\em center} of $A$ is the unique point $\bf{c}_A \in \P_\R^3$ such that $A \bf{c}_A = 0$. The camera $A$ is a rational map from the ``world'' $\P^3$ to the ``camera plane'' $\P^2$ sending a ``world point'' $\bf{q}$ to its ``image'' $A \bf{q}$. It is 
defined everywhere except at $\bf{c}_A$. Consider the hyperplane at infinity in $\P^3$, $L_\infty = \{ \bf{q} \in \P^3 \,:\, \bf{n}_\infty^\top \bf{q} = 0 \}$,  
as an oriented hyperplane in $\R^4$ with fixed normal 
$\bf{n}_\infty = (0,0,0,1)^\top$. The camera $A$ is said to be {\em finite} if $\bf{c}_A$ is a finite point, i.e., $\bf{c}_A \not \in L_\infty$.
A special representative of a camera center can be obtained by Cramer's rule where the $i$th coordinate of $\bf{c}_A$ is the determinant of the 
submatrix of $A$ obtained by dropping the $i$th column. 
In particular, for a finite camera $A = \begin{bmatrix} G & \bf{t}\end{bmatrix}$, the Cramer's rule center is $\bf{c}_A = \det(G) (-G^{-1} \bf{t}, 1)^\top$. Throughout this paper we use the Cramer's rule representation of $\bf{c}_A$. The camera $A = \begin{bmatrix} G & \bf{t}\end{bmatrix}$ is finite if and only if $\det(G) \neq 0$, and all cameras in this paper will be finite. 

The {\em principal plane} of a finite camera $A= \begin{bmatrix} G & \mathbf{t} \end{bmatrix}$ is the hyperplane $L_A := \{ \mathbf{q} \in \P^3 \,:\, A_{3,\bullet} \mathbf{q} = 0 \},$ where $A_{3,\bullet}$ is the third row of $A$, i.e. it is the set of points in $\P^3$ that image to infinite points in $\P^2$. Note that the camera center $\mathbf{c}_A$ lies on $L_A$. We regard $L_A$ as an oriented hyperplane in $\R^4$ with normal vector $\mathbf{n}_A := \det(G) A_{3 \bullet}^\top$, which we call the \emph{principal ray} of $A$. The $\det(G)$ factor ensures that the normal vector of the principal plane does not flip sign under a scaling of $A$. 
The depth of a finite point $\mathbf{q}$ in a finite camera $A$ is 
defined as (see \cite{HartleyZisserman2004})
\begin{equation}
\operatorname{depth}(\mathbf{q};A) := 
\left(\frac{1}{|\det(G)|\|G_{3,\bullet}\|} \right) \frac{( \mathbf{n}_A^\top \mathbf{q})}{(\mathbf{n}_\infty^\top \mathbf{q})}. \label{def:depth}
\end{equation}
Note that the sign of $\operatorname{depth}(\mathbf{q};A)$ is unaffected by scaling  $\mathbf{q}$ and $A$. 
The depth of a finite point $\mathbf{q} \in \P^3$ in a finite camera $A$ defined in \Cref{def:depth} is zero if and only if $\mathbf{n}_A^\top \mathbf{q}=0$ which 
happens if and only if $\mathbf{q}$ lies on the principal plane $L_A$. Otherwise, $\mathbf{n}_A^\top \mathbf{q} \neq 0$ and $\sign(\operatorname{depth}(\mathbf{q};A)) = \sign((\mathbf{n}_A^\top \mathbf{q})(\mathbf{n}_\infty^\top \mathbf{q}))$ is either positive or negative. 
It is then natural to say that a finite point $\mathbf{q}$ is {\em in front of} the camera $A$ if $\operatorname{depth}(\mathbf{q};A) > 0$, see~\cite{hartley1998chirality}. Since only the sign of 
$\operatorname{depth}(\mathbf{q};A)$ matters, we refer to this sign as the {\em chirality} of $\mathbf{q}$ in $A$, denoted as $\chi(\mathbf{q};A)$, which is either $1$ or $-1$. 

The above notion of chirality was introduced by Hartley in the seminal paper \cite{hartley1998chirality}, where he was concerned with a pair of cameras, see also \cite[Chapter 21]{HartleyZisserman2004}. 
In \cite{APST2020multiview}, the definition of chirality was extended to all points in $\P^3$, finite and infinite, and 
defined for an arrangement of cameras. Here is the two camera version we need.

\begin{definition}
 \label{def:chiral domain}
 Let $(A_1,A_2)$ be a pair of finite projective cameras. Then the \emph{chiral domain} of $(A_1,A_2)$, is the Euclidean closure in $\P^3$ of the set
  \[
 \{\mathbf{q} \in \P^3 \,|\, \mathbf{q} \text{ finite},\,\,  
 \operatorname{depth}(\mathbf{q}, A_1) > 0, \,\,
 \operatorname{depth}(\mathbf{q}, A_2) > 0 \}.
 \]
A point $\mathbf{q} \in \P^3$ is said to have \emph{chirality 1 with respect to }$(A_1,A_2)$, denoted as  $\chi(\mathbf{q} ; (A_1,A_2)) = 1$, if and only if $\mathbf{q}$ lies in the chiral domain of $(A_1,A_2)$.
\end{definition}

In this paper we will be concerned with a pair of finite {\em non-coincident} cameras $(A_1,A_2)$ by which we mean that their centers are distinct. 
We will see that one can always take $A_1 = \begin{bmatrix} I & \bf{0} \end{bmatrix}$, and then the conditions of finite and non-coincident imply that 
$A_2 = \begin{bmatrix} G & \bf{t} \end{bmatrix}$ where $G \in \textup{GL}_3$ and $\bf{t} \neq \bf{0}$. The pair $(A_1,A_2)$ 
gives rise to the unique (up to scale) real, rank two \emph{fundamental matrix} $X = [\bf{t}]_\times G$ where 
\begin{align*}
     [ \mathbf{t} ]_\times = \begin{bmatrix} 0 & -t_3 & t_2 \\ 
     t_3 & 0 & -t_1 \\ -t_2 & t_1 & 0 \end{bmatrix}.
\end{align*}
The skew-symmetric matrix $[\bf{t}]_\times$ represents the cross product with $\bf{t}$ as a linear map; that means that 
for $\bf{t}, \bf{r} \in \R^3$ we have
 $\mathbf{t} \times \mathbf{r} = [\mathbf{t}]_\times \mathbf{r} = 
 [\mathbf{r}]_\times^\top \mathbf{t}$. Also, $\rank([\bf{t}]_\times) = 2$ if and 
 only if $\bf{t} \neq \bf{0}$. 
 We are only ever interested in properties of fundamental matrices (mostly rank and kernels) that remain unaffected by scaling and will therefore mostly consider them up to scaling which is to say as a point of $\P(\R^{3\times 3})$. We chose not to introduce a different notation to distinguish the matrix itself from the line it spans.

The {\em epipole pair} of the cameras $(A_1, A_2)$ is 
$(\bf{e}_1, \bf{e}_2) \in \P^2 \times \P^2$ where $\bf{e}_1$ is the image of the center $\bf{c}_2$ in $A_1$, and $\bf{e}_2$ is the image of the center $\bf{c}_1$ in $A_2$. The line joining the 
centers $\bf{c}_1$ and $\bf{c}_2$ is called the {\em baseline} of the camera pair $(A_1, A_2)$. Note that all world points 
on the baseline (with the exception of the respective camera centers)  image to the epipole in each camera.

\section{Projective and Chiral Reconstructions}
\label{sec:reconstructions}


Throughout this paper our input is a collection of point pairs $\mathcal{P} = \{ (\bf{u}_i, \bf{v}_i) \,:\, i = 1, \ldots, k \} \subset \P_\R^2 \times \P_\R^2$. Each $\bf{u}_i$ (and $\bf{v}_i$) is the homogenization of a 
point in $\R^2$ by adding a last coordinate one. Hence $(\bf{u}_i,\bf{v}_i)$ is a pair of finite points in $\P_\R^2 \times \P_\R^2$ with 
a fixed representation. We will 
also assume that all $\bf{u}_i$ (and all $\bf{v}_i$) are distinct.

In this section we formally define projective and chiral reconstructions of $\mathcal{P}$, and characterize their existence. We then 
set up the geometric framework within which these reconstructions will be studied in this paper. 

\subsection{Projective reconstructions}
\begin{definition} \label{def:projective reconstruction}
A {\em projective reconstruction} of $\mathcal{P}$ consists of a pair of projective cameras 
$A_1$, $A_2 \in \P(\R^{3 \times 4})$,  
world points $\mathcal{Q} = \{ \mathbf{q}_1, \ldots, \mathbf{q}_k \} \subset 
\P^3$ and non-zero scalars $w_{1i}, w_{2i}$ such that $A_1 \mathbf{q}_i = w_{1i} \bf{u}_i$ and 
$A_2 \mathbf{q}_i = w_{2i} \bf{v}_i$ for $i=1, \ldots, k$. 
If the cameras and world points are all finite, then $(A_1, A_2, \mathcal{Q})$ is called a 
{\em finite projective reconstruction} of $\mathcal{P}$.
\end{definition}

The basics of projective reconstructions can be found in \cite[Chapters 9 \& 10]{HartleyZisserman2004}. Theorem 3.1 in \cite{Lee16} proves that if $\mathcal{P}$ has a projective reconstruction then it also has a finite projective reconstruction with $A_1 = \begin{bmatrix} I & 0 \end{bmatrix}$.

We now recall the necessary and sufficient conditions for the existence of a finite projective reconstruction of $\mathcal{P}$.  For a point 
$\bf{e} \in \P^2$, let $\bf{e}(\bf{u}_1, \ldots, \bf{u}_k)$ denote the set of lines joining $\bf{e}$ to each $\bf{u}_i$.
The following geometric characterization is well-known  \cite{HartleyZisserman2004,maybank, WernerPajdla2001,werner2003constraint}.

\begin{theorem} \cite[Section 9.4]{HartleyZisserman2004}, \cite[Section 2.4]{maybank}
\label{thm:epipolar homography} 
The set of point pairs $\mathcal{P}$ has a projective reconstruction $(A_1,A_2,\mathcal{Q})$ if and only if there exist points $\bf{e}_1, \bf{e}_2 \in \P^2$ and a homography which sends the set of $k$ lines $\bf{e}_1(\bf{u}_1, \bf{u}_2, \dots, \bf{u}_k)$ to the set of $k$ lines $\bf{e}_2(\bf{v}_1, \bf{v}_2, \dots, \bf{v}_k)$ where the line $\bf{e}_1\bf{u}_i$ maps to the line $\bf{e}_2\bf{v}_i$.
\end{theorem}

The points $\bf{e}_1, \bf{e}_2$ in the above theorem are the epipoles of the camera pair $(A_1, A_2)$ in the reconstruction. We now give a second characterization of the existence of a projective reconstruction in terms of \emph{fundamental matrices}: Below, we will abuse notation and not distinguish between the matrix and the line it spans in the space of matrices (or in other words the corresponding point in projective space). All relevant properties of the matrices are invariant under scaling.

\begin{definition}
\begin{enumerate}
    \item 
 A {\em fundamental matrix} of $\mathcal{P}$ is a rank two matrix $X \in \P(\R^{3 \times 3})$ such that 
 \begin{align}\label{eq:epipolar equations}
   \bf{v}_i^\top X \bf{u}_i = 0, \,\,\, \textup{ for } i=1, \ldots, k.  
 \end{align}
 The linear equations \eqref{eq:epipolar equations} in $X$ are called the {\em epipolar 
 equations} of $\mathcal{P}$. 
 
 \item Given a rank two matrix $X \in \P(\C^{3 \times 3})$ and a pair 
 $(\bf{u}, \bf{v}) \in \P^2 \times \P^2$, such that $\bf{v}^\top X \bf{u} = 0$, we say that $X$ is $(\bf{u}, \bf{v} )$-{\em regular} if 
$\bf{v}^{\top} X = 0$ if and only if $X\bf{u} = 0$. i.e., $\bf{u}$ and $\bf{v}$  
simultaneously generate the right and left kernels of $X$, or neither generate a kernel.
\item A \emph{$\mathcal{P}$-regular fundamental matrix} is a fundamental matrix of $\mathcal{P}$ that is $(\bf{u}_i, \bf{v}_i)$-regular for each point pair in $\mathcal{P}$.
\end{enumerate}
\end{definition}


It is commonly believed that $\mathcal{P}$ has a projective 
reconstruction if and only if it has a fundamental matrix. However, a bit more care is needed as in the following theorem. 

 \begin{theorem} \cite[Theorem 4.6]{Lee16} 
 \label{thm:equivalentprconditions}
 There exists a finite projective reconstruction of $\mathcal{P}$ with 
 two non-coincident cameras, and $A_1 =\begin{bmatrix} I & \mathbf{0}\end{bmatrix}$,  if and only if there exists a $\mathcal{P}$-regular fundamental matrix.

 \end{theorem}

\begin{remark} 
\begin{enumerate}
\item Let $A_1 =\begin{bmatrix} I & \mathbf{0}\end{bmatrix}$ and $A_2 = \begin{bmatrix} G & \bf{t} \end{bmatrix}$ be the two finite non-coincident cameras in the projective reconstruction. Then recall that $G \in \textup{GL}_3$,  $\bf{t} \neq \bf{0}$, and the cameras correspond to the unique fundamental matrix $X =  [\bf{t}]_\times G$ up to scale. The epipoles of the cameras are $\bf{e}_1 \sim  G^{-1} \bf{t}$
and $\bf{e}_2 \sim \bf{t}$ which generate the right and left kernels of 
$X$. Further, $G$ defines a homography of 
$\P^2$ that sends $\bf{e}_1$ to $\bf{e}_2$, and the 
line  $\bf{e}_1\bf{u}_i$ to the line $\bf{e}_2\bf{v}_i$. Hence $G$ encodes the epipolar line homography of \Cref{thm:epipolar homography}. \label{rem:item1}

\item Conversely, any (rank two) fundamental matrix $X$ of $\mathcal{P}$ can be factored as $X = [\bf{t}]_\times G$ for some $\bf{t} \in \R^3$ and 
$G \in \textup{GL}_3$ and yield a pair of cameras $(A_1 =\begin{bmatrix} I & \mathbf{0}\end{bmatrix}, A_2 = \begin{bmatrix} G & \bf{t} \end{bmatrix})$ whose 
epipoles generate the left and right kernels of $X$. See \cite[Section 9.5]{HartleyZisserman2004} for more details on the correspondence between fundamental matrices and camera pairs.
These cameras can then be used to reconstruct a set of world points $Q$ if $X$ is $\mathcal{P}$-regular. The resulting projective reconstruction is said to be 
{\em associated to} $X$. 

\item A fundamental matrix 
is $\mathcal{P}$-regular if and only if 
for each $i$, either $\bf{u}_i$ and $\bf{v}_i$ are both epipoles of the cameras 
or neither are. Indeed, this is a necessary condition for a projective reconstruction since if $\bf{u}_i \sim \bf{e}_1$, then its 
reconstruction $\bf{q}_i \in \P^3$ lies on the baseline of the cameras 
and hence images to $\bf{e}_2$ in camera $A_2$ requiring 
$\bf{v}_i \sim \bf{e}_2$. 
This subtlety is often overlooked, and it is common to equate the existence of a projective reconstruction of $\mathcal{P}$ to the existence of a fundamental matrix of $\mathcal{P}$. 

\item Lastly, we remark that \cite[Theorem 4.6]{Lee16} is stated using a different notion of regularity. However, a fundamental matrix $X$ is $\mathcal{P}$-regular in our sense if and only if it is $([\bf{t}]_\times X, \bf{t})$-regular in the sense of \cite{Lee16} and hence the above theorem is exactly \cite[Theorem 4.6]{Lee16}.
\end{enumerate}
\end{remark}

We now discuss the geometry encoded in Theorem~\ref{thm:equivalentprconditions} which will set the foundation for the work in this paper. Even though fundamental matrices are real, we will work over $\C$ to allow for methods from complex algebraic geometry, and will specialize to $\R$ as needed. A matrix $X \in \P(\C^{3 \times 3})$ can be identified with a point in $\P^8$ by concatenating its rows. Under this identification we let $\mathcal{R}_2 \subset \P^8$ be the determinantal hypersurface of matrices in $\P(\C^{3 \times 3})$ of rank at most two, and $\mathcal{R}_1$ be its subvariety of rank one matrices. As projective subvarieties of $\P^8$, $\dim \mathcal{R}_2 = 7$ and $\textup{degree} \,\,\mathcal{R}_2  = 3$ while, $\dim \mathcal{R}_1 = 4$ and $\textup{degree} \,\,\mathcal{R}_1 = 6$. For a point pair $(\bf{u}, \bf{v}) \in \P^2 \times \P^2$, let $L_{(\bf{u},\bf{v})}$ denote the hyperplane in $\P^8$:
 \begin{align*}
  L_{(\bf{u},\bf{v})} = \{ X \in \P^{8} \,:\, \bf{v}^\top X \bf{u}  = 
  \langle X, \bf{v} \bf{u}^\top \rangle := \textup{Tr}(X^\top \bf{v} \bf{u}^\top) = 0 \}
 \end{align*}
 where $\langle \cdot, \cdot \rangle$ denotes the Frobenius inner product on matrices. Let $\mathcal{L}_\mathcal{P} = \bigcap_{i=1}^k  L_{(\bf{u}_i, \bf{v}_i)}$. Generically, $\mathcal{L}_\mathcal{P}$ is a linear space in $\P^8$ of codimension $k$.
 
 \begin{definition} The variety $\mathcal{R}_2 \cap \mathcal{L}_\mathcal{P}$ in $\P^8$ is the 
  {\em epipolar variety} of $\mathcal{P}$. 
 \end{definition}
Under sufficient genericity of $\mathcal{P}$, $\dim(\mathcal{R}_2 \cap \mathcal{L}_\mathcal{P})  = 7-k$ and  $\textup{degree}(\mathcal{R}_2 \cap \mathcal{L}_\mathcal{P}) = 3$. Hence, the epipolar variety is empty when 
$k \geq 8$, consists of three points when $k=7$, and infinitely many 
points when $k < 7$. 
  
 For $(\bf{u}_i, \bf{v}_i) \in \mathcal{P}$, consider the following five-dimensional linear spaces of $\P^8$ that are in fact in $\RR_2$:
 $$\widetilde{W}_{\bf{u}_i} := \{ X \in \P^{8} \,:\, X \bf{u}_i = 0 \} 
 \,\,\,\textup{ and } \,\,\,\widetilde{W}^{\bf{v}_i} := \{ X \in \P^{8} \,:\, {\bf{v}_i}^\top X = 0 \}.$$ 
 Their intersections with the epipolar variety are the 
 linear spaces:
 $W_{\bf{u}_i} := \mathcal{L}_\mathcal{P} \cap \widetilde{W}_{\bf{u}_i}$ and 
 $W^{\bf{v}_i} := \mathcal{L}_\mathcal{P} \cap \widetilde{W}^{\bf{v}_i}$, 
 each of which generically has dimension $6-k$ since $\widetilde{W}_{\bf{u}_i}, 
 \widetilde{W}^{\bf{v}_i} \subset L_{(\bf{u}_i, \bf{v}_i)}$. 
 
 \begin{definition} \label{def:walls and corners} 
 The linear space $W_{\bf{u}_i}$ (resp. $W^{\bf{v}_j}$) 
will be called the $\bf{u}_i$ (resp. $\bf{v}_j$) {\em wall} and the intersection $W_{\bf{u}_i} \cap W^{\bf{v}_j}$ will be called 
the $(\bf{u}_i, \bf{v}_j)$ {\em corner}. 
\end{definition}

When $i \neq j$,  a $(\bf{u}_i, \bf{v}_j)$ corner 
has dimension $5-k$ generically since there are at most $5$ independent equations among $\bf{v}_j^\top X = 0 = X \bf{u}_i$. Thus a wall has codimension one and a corner has codimension two in the epipolar variety, generically. 
For non-generic data $\mathcal{P}$, all the dimensions computed above may be larger.

The second condition in \Cref{thm:equivalentprconditions} can now be rephrased as the existence of a real rank two matrix $X$ in 
the epipolar variety of $\mathcal{P}$ such that for each $i$, $X$ is either in the $(\bf{u}_i, \bf{v}_i)$ corner or in the complement of $W_{\bf{u}_i} \cup W^{\bf{v}_i}$.
Note that a rank two $X$ can lie in at most one $(\bf{u}_i, \bf{v}_i)$ corner because the points $\vu_i$ (and $\vv_j$) are pairwise distinct.

Going forward, we will work both in $\P^8$, the space of fundamental matrices, and in $\P^2 \times \P^2$, the space of epipoles. These spaces are related by the {\em adjoint map}, 
\[
\adj : \P^8 \dashrightarrow \P^8, \quad X \mapsto \adj(X)
\]
where $\adj(X) = \cof(X)^\top$ and $\cof(X)$ is the cofactor matrix of $X$. If $X \in \RR_2$ then $X \cdot \textup{adj}(X) = \textup{adj}(X) \cdot X = 0$ and thus, if $\rank(X)=2$, then all non-zero rows (resp.~columns) of $\adj(X)$ are multiples of each other and generate the left (resp.~right) kernel of $X$. Since generators of the right and left kernels of a fundamental matrix represent epipoles, the adjoint map provides a convenient connection between epipole space and fundamental matrix space.

\subsection{Chiral reconstructions}
A physical constraint on a true reconstruction 
$(A_1,A_2, \mathcal{Q})$ is that the reconstructed world points
in $\mathcal{Q}$ must lie in front of the cameras $A_1$ and $A_2$. Recall from the Introduction that this means we require 
$\mathcal{Q}$ to lie in the chiral domain of $(A_1,A_2)$ or equivalently, $\chi(\bf{q}_i; (A_1, A_2))=1$ for all $\bf{q}_i \in \mathcal{Q}$. 
A full development of multiview chirality can be found in \cite{APST2020multiview}. For this paper, we use the following inequality description of the chiral domain for two views from \cite[Theorem 1]{APST2020multiview} as a definition. 
The cited result shows that these inequalities cut out the Euclidean closure of the set in \Cref{def:chiral domain} (under the mild assumption that it has non-empty interior).

\begin{definition}
\label{def:chiral reconstruction}
A \emph{chiral reconstruction} of $\mathcal{P}$ is a projective reconstruction $(A_1, A_2, \mathcal{Q})$ of $\mathcal{P}$ with finite non-coindent cameras such that for all $i$,
\[
(\mathbf{n}_\infty^\top \mathbf{q}_i) (\mathbf{n}_1^\top \mathbf{q}_i)   \geq 0, \ (\mathbf{n}_\infty^\top \mathbf{q}_i) (\mathbf{n}_2^\top \mathbf{q}_i) \ge 0 , \text{ and } \
(\mathbf{n}_1^\top \mathbf{q}_i)(\mathbf{n}_2^\top \mathbf{q}_i)\geq 0
\]
where $\bf{n}_\infty = (0,0,0,1)^{\top}$ and  $\mathbf{n}_i$ is the principal ray of $A_i$.
\end{definition}

 Recall that two projective reconstructions $(A_1, A_2, \mathcal{Q})$ and $(A_1', A_2', \mathcal{Q}')$ are 
 {\em projectively equivalent} if they are 
 related by a homography of $\P^3$, i.e., there is a $H \in \GL_4$ such that $A_i' = A_i H^{-1}$ and $\mathcal{Q}' = H \mathcal{Q} := \{ H\mathbf{q}_i, \,i=1,\ldots,k \}$. A projective reconstruction which is not chiral can sometimes be transformed into a chiral reconstruction by a homography \cite{hartley1998chirality, APST2020multiview,  WernerPajdla2001}. We recall the conditions under which 
 this is possible. 
\begin{theorem}
\label{thm:equivalentchiralconditions}
Consider a finite projective reconstruction of $\mathcal{P}$ with 
 non-coincident cameras $A_1 =\begin{bmatrix} I & \mathbf{0}\end{bmatrix}$, $A_2 = \begin{bmatrix} G & \mathbf{t}\end{bmatrix}$, and world points $\mathcal{Q} = \{\mathbf{q}_1, \ldots, \mathbf{q}_k \} \subset \P^3$. Then the following are equivalent.

\begin{enumerate}
    \item There exists a projectively equivalent chiral reconstruction $(A_1H^{-1} , A_2 H^{-1} , H\mathcal{Q})$ of $\mathcal{P}$. \label{item:exists chiral}
    \item $(\bf{n}_1^\top \bf{q}_i)(\bf{n}_2^\top \bf{q}_i)$ has  the same sign for all $i$. \label{item:principal rays}
    \item $w_{1i}w_{2i}$ has the same sign for all $i$. \label{item:wiwi'}
\end{enumerate}
Furthermore, if no $\bf{q}_i$ lies on the baseline of 
$(A_1, A_2)$, then (\ref{item:exists chiral}), (\ref{item:principal rays}),  (\ref{item:wiwi'}) are equivalent to 
\begin{enumerate}\setcounter{enumi}{3}
    \item $(\bf{t} \times \bf{v}_i )^\top( \bf{t}  \times G \bf{u}_i )$  has the same sign for all $i$. \label{item:cji}
\end{enumerate}

\end{theorem}

\begin{proof}
The equivalence of (\ref{item:exists chiral}) and (\ref{item:principal rays}) is Theorem 8 in \cite{APST2020multiview}. The equivalence of (\ref{item:exists chiral}) and (\ref{item:wiwi'}) is Theorem 17 in \cite{hartley1998chirality}. 
The epipoles of the given cameras are  
$\bf{e}_1 \sim G^{-1}\bf{t}$ and 
$\bf{e}_2 \sim \bf{t}$, and hence if no world point lies on the baseline, $G^{-1}\bf{t} \not \sim \bf{u}_i$ (equivalently, $\bf{t} \not \sim G\bf{u}_i$) and $\bf{t} \not \sim  \bf{v}_i$ for any $i$. Also, since $A_1, A_2$ are non-coincident and $\bf{t} \neq \bf{0}$, \Cref{thm:epipolar homography} and \Cref{rem:item1} imply that $\bf{t}, \bf{v}_i, G\bf{u}_i$ are  collinear. Therefore, 
$(\bf{t} \times \bf{v}_i ) \sim ( \bf{t}  \times G \bf{u}_i )$ and so 
$(\bf{t} \times \bf{v}_i )^\top( \bf{t}  \times G \bf{u}_i ) \neq 0$ for all $i$. The equivalence of (\ref{item:wiwi'}) and (\ref{item:cji}) can then be derived from the same arguments as in Lemma 7 in \cite{APST2020multiview}. 
\end{proof}

If a world point lies on the baseline, then its images $(\bf{u}_i, \bf{v}_i)$ in 
$A_1,A_2$ are the epipoles of the cameras, and the expression in (\ref{item:cji}) becomes zero. However, since any point on the baseline can serve as the world point $\bf{q}_i$ imaging to $(\bf{u}_i, \bf{v}_i)$,  we can control the sign of 
$(\mathbf{n}_1^\top \mathbf{q}_i)(\mathbf{n}_2^\top \mathbf{q}_i)$ as shown next. 

\begin{lemma}
\label{lem:replace baseline point}
For a pair of non-coincident cameras $(A_1, A_2)$ whose baseline is not contained in either principal plane $L_{A_1}$ and $L_{A_2}$, there exist $\mathbf{q}_+$ and $\mathbf{q}_-$ on the baseline such that $ (\mathbf{n}_1^\top \mathbf{q}_+)(\mathbf{n}_2^\top \mathbf{q}_+) >0$ and $ (\mathbf{n}_1^\top \mathbf{q}_-)(\mathbf{n}_2^\top \mathbf{q}_-) < 0$.
\end{lemma}

\begin{proof}
Since $\bf{c}_1 \in L_{A_1}$, $\bf{n}_1^\top \bf{c}_1 = 0$. On the other hand, 
since the baseline is not contained in $L_{A_2}$ and $\bf{c}_1 \neq \bf{c}_2$, $\bf{c}_1 \notin L_{A_2}$, 
and so $(\bf{n}_2^\top \bf{c}_1) \neq 0$. By continuity, there exist perturbations $\bf{q}_+$ and $\bf{q}_-$ of $\bf{c}_1$ on the baseline such that $ (\mathbf{n}_1^\top \mathbf{q}_+)(\mathbf{n}_2^\top \mathbf{q}_+) >0$ and $ (\mathbf{n}_1^\top \mathbf{q}_-)(\mathbf{n}_2^\top \mathbf{q}_-) < 0$.
\end{proof}

\begin{remark}
For reconstructions where both epipoles 
are finite, the hypothesis of \Cref{lem:replace baseline point} is satisfied. Indeed, if for instance the baseline was contained in the principal plane $L_{A_1}$ then $\bf{c}_2 \in L_{A_1}$ and so $\bf{e}_1 \sim A_1 \bf{c}_2$ would be an infinite point.
\end{remark}

We now have a necessary and sufficient condition for the existence of a chiral 
reconstruction.

\begin{lemma}
\label{lem:reconstruct from chiral joint image}
There exists a chiral reconstruction of $\mathcal{P}$ if and only if there exist $\bf{t} \in \R^3 \setminus \{\bf{0}\}$ and $G \in GL_3$ such that $[\bf{t}]_\times G$ is a $\mathcal{P}$-regular fundamental matrix and
 \begin{align}
\label{eq:chiral joint image nonstrict} \left((\bf{t} \times \bf{v}_i )^\top( \bf{t}  \times G \bf{u}_i )\right)\left( ( \bf{t} \times \bf{v}_j )^\top( \bf{t}  \times G \bf{u}_j )\right) \ge 0 \,\, \text{ for all } \,\,1 \le i< j \le k.
 \end{align}
\end{lemma}

\begin{proof}
Suppose $(A_1, A_2, \mathcal{Q})$ is a chiral reconstruction of $\mathcal{P}$ 
with non-coincident finite cameras where $A_1 = \begin{bmatrix} I & \bf{0}\end{bmatrix}$. Then  $A_2 =\begin{bmatrix} G& \bf{t}\end{bmatrix}$ for some $\bf{t} \in \R^3 \setminus \{\bf{0}\}$ and $G \in GL_3$, and by \Cref{thm:equivalentprconditions}, $[\bf{t}]_\times G$ is a $\mathcal{P}$-regular fundamental matrix associated to $A_1$ and $A_2$. 
For all $i$ such that $\bf{q}_i$ is not on the baseline, $(\bf{t} \times \bf{v}_i )^\top( \bf{t}  \times G \bf{u}_i )$ has the same sign by \Cref{thm:equivalentchiralconditions}. If some world point $\bf{q}_i$ is on the baseline, then its image $(\bf{u}_i, \bf{v}_i)$ is the pair of epipoles $(G^{-1}\bf{t}, \bf{t})$, and  hence $(\bf{t} \times \bf{v}_i )^\top( \bf{t}  \times G \bf{u}_i) = 0$. Therefore, the inequalities in (\ref{eq:chiral joint image nonstrict}) hold. 

Conversely, suppose there exist $\bf{t} \in \R^3 \setminus \{\bf{0}\}$ and $G \in GL_3$ such that $[\bf{t}]_\times G$ is a $\mathcal{P}$-regular fundamental matrix and the inequalities (\ref{eq:chiral joint image nonstrict}) hold. By \Cref{thm:equivalentprconditions}, there exist world points $\mathcal{Q}$ such that $(A_1 = \begin{bmatrix} I & \bf{0}\end{bmatrix}, A_2 = \begin{bmatrix} G& \bf{t}\end{bmatrix}, \mathcal{Q})$ is a projective reconstruction of $\mathcal{P}$. Let $\widehat{\mathcal{Q}} \subseteq \mathcal{Q}$ be the set of world points not on the baseline of $A_1$ and $A_2$. Since the inequalities (\ref{eq:chiral joint image nonstrict}) hold, the quadruple products $(\bf{t} \times \bf{v}_i )^\top( \bf{t}  \times G \bf{u}_i)$ have the same sign for all $\bf{q}_i \in \widehat{\mathcal{Q}}$. By \Cref{thm:equivalentchiralconditions}, there exists a constant $\si \in \{-1,1\}$ such that $\si = \sign (\bf{n}_1^\top \bf{q}_i)(\bf{n}_2^\top \bf{q}_i)$ for all $\bf{q}_i \in \widehat{\mathcal{Q}}$.

If some point $\bf{q}_j \in \mathcal{Q}$ lies on the baseline, then $\bf{q}_j$ images to the pair of epipoles $G^{-1} \bf{t}$ and $\bf{t}$ in the two cameras and  hence $(\bf{t} \times \bf{v}_j )^\top( \bf{t}  \times G \bf{u}_j) = 0$. By \Cref{lem:replace baseline point}, we may replace $\bf{q}_j$ by some world point $\bf{q}_j'$ on the baseline such that $\sign(\bf{n}_1^\top \bf{q}_j')(\bf{n}_2^\top \bf{q}_j') = \si$. Let $\mathcal{Q}'$ be the modification of $\mathcal{Q}$ obtained 
by replacing all world points on the baseline as above, but keeping all other world points intact.
By construction, $\sign(\bf{n}_1^\top \bf{q}_i')(\bf{n}_2^\top \bf{q}_i') = \si$ for all $\bf{q}_i' \in \mathcal{Q}'$. The transformed reconstruction $(A_1, A_2, \mathcal{Q}')$ is projectively equivalent to a chiral reconstruction by \Cref{thm:equivalentchiralconditions}.
\end{proof}

\Cref{lem:reconstruct from chiral joint image} implies that for a chiral reconstruction to exist, there must be $\mathcal{P}$-regular fundamental matrices that satisfy the 
inequalities \eqref{eq:chiral joint image nonstrict}. In the next section, we examine these inequalities to understand the regions of the epipolar variety in which fundamental matrices that lead to chiral reconstructions live.

\section{Chiral tools}
\label{sec:chiraltools}

In this section we develop tools to 
prove the existence of chiral reconstructions. In \Cref{subsec:ineq matrix space}, we describe the semialgebraic {\em chiral epipolar region} of fundamental matrices associated to chiral reconstructions of $\mathcal{P}$. In \Cref{subsec:ineq epipole space}, we show how inequalities defining the chiral epipolar region can be checked in epipole space. Even if not stated explicitly, we are working over $\P_\R$ 
when dealing with inequalities. We combine these tools in \Cref{subsec:three pairs} to prove that a set of three point pairs has a chiral reconstruction. In \Cref{subsec:chiral boundary}, we show how the walls and corners of the epipolar variety can be used to decide if a set of more than three point pairs has a chiral reconstruction.

\subsection{The chiral epipolar region}
\label{subsec:ineq matrix space}

By \Cref{lem:reconstruct from chiral joint image}, a fundamental matrix must satisfy \eqref{eq:chiral joint image nonstrict} to yield a chiral reconstruction. In this section, we describe the strict subset of the epipolar variety 
satisfying these constraints.

\begin{definition}
\label{def:gi}
Let $X \in \P(\R^{3\times 3})$ denote a real $3\times 3$-matrix up to scaling. For each point pair $(\bf{u}_i, \bf{v}_i)$, the \emph{$i$-th chiral polynomial} is
\[
g_i(X) := \bf{v}_i^\top [-\bf{t}]_\times X \bf{u}_i = (\bf{t}\times \bf{v}_i)^\top X \bf{u}_i
\]
where $\bf{t}^\top X = \bf{0}$. The set of all $g_ig_j (X) = g_i(X) g_j (X) \ge 0$ are called the \em chiral epipolar inequalities of $\mathcal{P}$. Here, the same representative $\bf{t}$ for the left-kernel of $X$ must be used in $g_i$ and $g_j$.
\end{definition}

The $i$-th chiral polynomial is, strictly speaking, not a polynomial because there is no way to write a generator of the left kernel of a matrix $X$ as a polynomial expression that works for every $3\times 3$ matrix of rank $2$. To be technically precise, it 
is a section of a line bundle on the quasi-projective variety of $3\times 3$ matrices of rank exactly two. We avoid these technicalities and argue that the chiral epipolar inequalities are well-defined in an elementary way using the adjoint: Writing $\bf{t}$ with $\bf{t}^\top X = \bf{0}$ in terms of $X$ is the composition of the adjoint map $\adj\colon \RR_2 \setminus \RR_1 \to \P^8$, whose image is $\P^2\times \P^2$ in its Segre embedding, with the projection $\P^2\times \P^2 \to \P^2$.  So locally, $\bf{t}$ is given as a row of the adjoint matrix $\adj(X)$ (but only on the open set where that row is non-zero). The entries of the matrix $[-\bf{t}]_\times X$ are polynomials of degree three in the entries of $X$. This shows that the inequalities $g_ig_j(X)\geq 0$ are locally of degree six in the entries of $X$. 
Also, if two rows $\bf{t}$ and $\bf{t}'$ of $\adj(X)$ are non-zero and differ by a negative multiple $\lambda\in \R_{<0}$, i.e.~$\bf{t} = \lambda \bf{t}'$, the sign of $g_ig_j(X)$ does not change because it essentially differs by $\lambda^2$.
Therefore the sign of $g_ig_j$ is well-defined for every real $3\times 3$ matrix of rank two up to scaling, i.e.~for every fundamental matrix.
The set of real matrices $X$ in $\RR_2$ for which $g_ig_j(X) \geq 0$ is a semi-algebraic subset of $\P(\R^{3\times 3})$ 
in the following sense: There is an open affine cover of $\RR_2\setminus\RR_1$ (by sets on which we can write $\bf{t}$ as a polynomial function of $X$), such that the inequalities $g_ig_j(X) \geq 0$ become polynomial and hence define a semi-algebraic set in each open subset of (the real points in) this cover. On the intersection of any two open sets in the cover, the regions cut out by these inequalities agree.

We now show that the chiral polynomial $g_i(X)$ records the quadruple product $(\bf{t} \times \bf{v}_i )^\top(\bf{t} \times G\bf{u}_i )$ from 
\Cref{lem:reconstruct from chiral joint image}. 

\begin{lemma}
\label{lem:gigj is quad product}
If $X = [\bf{t}]_\times G$, then  $g_i(X) = (\bf{t} \times \bf{v}_i )^\top( \bf{t}  \times G \bf{u}_i )$ for each $i$.
\end{lemma}

\begin{proof}
$g_i(X) = \bf{v}_i^\top [-\bf{t}]_\times X \bf{u}_i = ( [-\bf{t}]^\top_\times \bf{v}_i)^\top [\bf{t}]_\times G \bf{u}_i = ( [\bf{t}]_\times \bf{v}_i)^\top [\bf{t}]_\times G \bf{u}_i = (\bf{t}\times \bf{v}_i)^\top (\bf{t}\times G\bf{u}_i).
$
\end{proof}





The next theorem, which is analogous to Theorem 3 in \cite{WernerPajdla2001}, now follows from \Cref{lem:reconstruct from chiral joint image} and \Cref{lem:gigj is quad product}. 

\begin{theorem}
\label{thm:chiral F existence gigj}
 There exists a chiral reconstruction of $\mathcal{P}$ 
 if and only if there exists a $\mathcal{P}$-regular fundamental matrix $X$ such that $g_ig_j(X)  \ge  0 $ for all $1\le i< j \le k$ .
 \end{theorem}

\begin{definition}
 \label{def:chiral region}
The \emph{chiral epipolar region of $\mathcal{P}$} is the set of $\mathcal{P}$-regular fundamental matrices $X$ such that $g_ig_j(X)  \ge  0 $ for all $1\le i< j \le k$.
\end{definition}

The chiral epipolar region of $\mathcal{P}$ is contained in the semialgebraic subset of the real part of the epipolar variety $\RR_2 \cap \mathcal{L}_\mathcal{P}$ that is cut out by the chiral epipolar inequalities. It is not necessarily equal to this set because the chiral epipolar region additionally requires the fundamental matrices to be $\mathcal{P}$-regular. However, since $\PP$-regularity only fails on a proper algebraic subset, if the chiral epipolar region has non-empty interior, the boundary of the interior is determined by the points where the chiral epipolar inequalities change sign, which we study next.

\begin{lemma}
\label{lem:chiral zeros}
Let $X$ be a fundamental matrix of $\mathcal{P}$. Then $g_i(X) = 0$ if and only if $X \in W_{\bf{u}_i}$ or $X \in W^{\bf{v}_i}$. 
\end{lemma}

\begin{proof}
Clearly, $g_i(X) = \vv_i^\top [-\vt]_\times X \vu_i = 0$ if $X\vu_i = 0$. If $\vv_i^\top X=0$, then $\vv_i$ and $\vt$ are collinear and therefore $\vv_i [-\vt]_\times = \vnull$, which implies $g_i(X) = 0$. For the other implication, we know that $\vv_i^\top X \vu_i = 0$ and $\vv_i^\top [-\vt]_\times X\vu_i = 0$, where the three vectors $\vv_i$, $\vu_i$, and $\vt$ are real and non-zero. We assume that $X\vu_i \neq \vnull$ and show that $\vv_i^\top X = \vnull$. We know $X\vu_i$ is orthogonal to $\vv_i$ and $-\vt\times \vv_i$. Therefore, it must be collinear with $\vv_i\times (\vv_i \times \vt)$, which is the same as $(\vv_i^\top \vt) \vv_i - (\vv_i^\top \vv_i) \vt$. We also know that $\vt^\top X = 0$, which implies that $\vt$ is also orthogonal to $X\vu_i$, hence also to $\vv_i\times (\vv_i\times \vt)$. The dot product $\vt^\top \left( (\vv_i^\top \vt) \vv_i - (\vv_i^\top \vv_i) \vt \right) = 0$, i.e.~$(\vv_i^\top \vt)^2 = (\vv_i^\top \vv_i) (\vt^\top \vt)$. The Cauchy-Schwarz inequality implies that $\vt$ and $\vv_i$ are collinear, which implies the claim $\vv_i^\top X = \vnull$.
\end{proof}

The goal of the paper is to understand when the chiral epipolar region of $\mathcal{P}$ is non-empty, or equivalently, when $\mathcal{P}$ has a chiral reconstruction. When $k=7$, generically $\RR_2 \cap \mathcal{L}_\mathcal{P}$ consists of three points and it is easy to check if the real points lie in the chiral epipolar region of $\mathcal{P}$. Therefore, our focus will be on values of $k < 7$.

\subsection{Translating to epipole space}
\label{subsec:ineq epipole space}
In this section, we show how we can check the validity of chiral epipolar inequalities in 
$\P_\R^2 \times \P_\R^2$, the space of epipoles. Consider the $(1,1)$-homogeneous quadratic polynomial 
\begin{align}
\label{eq:Dij}
    D_{ij}(\bf{u},\bf{v}) := \det\begin{bmatrix} \bf{u}_i & \bf{u}_j & \bf{u} \end{bmatrix}\det\begin{bmatrix} \bf{v}_i & \bf{v}_j & \bf{v} \end{bmatrix}
\end{align}
where 
$(\bf{u}, \bf{v}) \in \P_\R^2 \times \P_\R^2$. 
Note that 
$D_{ij}(\bf{u},\bf{v}) = 0$ if and only if either factor is zero which is if and only if $\bf{u}_i,\bf{u}_j, \bf{u}$ are 
collinear or $\bf{v}_i, \bf{v}_j, \bf{v}$ are collinear. Werner uses the quantities $D_{ij}(\bf{u}, \bf{v})$ to impose an orientation constraint on the epipolar line homography described in our \Cref{thm:epipolar homography}, see \cite[Section 5]{werner2003constraint} and 
\cite[Section 6.5]{werner2003combinatorial}. We will show that $D_{ij}(\bf{u}, \bf{v})$ is closely related to the products of chiral polynomials $g_ig_j(X)$ where $\bf{u}$ and $\bf{v}$ generate the right and left kernels of a fundamental matrix $X$. We rely on the following well known identity \cite{chum2003joint, werner2003combinatorial}.

\begin{lemma}
\label{lem:3det to 4det}
Suppose $\bf{q}_1, \bf{q}_2, \bf{q}_3 \in \R^4$. Let $A = \begin{bmatrix} G & \mathbf{t}\end{bmatrix}$ be a finite camera with Cramer's rule center $\bf{c}_A = \det(G)(-G^{-1} \bf{t}, 1)$. Then $ \det\begin{bmatrix} A \bf{q}_1 & A \bf{q}_2 & A \bf{q}_3 \end{bmatrix} = \det\begin{bmatrix} \bf{q}_1 & \bf{q}_2 & \bf{q}_3 & \bf{c}_A\end{bmatrix}$. 
\end{lemma}

\begin{lemma}
\label{lem:wi to triple products}
Consider a projective 
reconstruction $(A_1, A_2,\mathcal{Q})$ of $\mathcal{P}$, i.e.,  
$A_1{\bf q}_i = w_{1i} \bf{u}_i$ and  
$A_2{\bf q}_i = w_{2i} \bf{v}_i$ where $w_{ij} \neq 0$. Suppose $ D_{ij}(-A_1\bf{c}_2 , A_2\bf{c}_1)  \neq 0$ where $\bf{c}_i$ is the Cramer's rule center of $A_i$. Then 
\[
\sign D_{ij}(-A_1\bf{c}_2 , A_2\bf{c}_1)  = \sign (w_{1i}w_{2i})(w_{1j}w_{2j}).
\]
\end{lemma}


\begin{proof}
Expand $D_{ij}(-A_1\bf{c}_2 , A_2\bf{c}_1)$ as follows. 
\begin{align}
   D_{ij}(-A_1\bf{c}_2 , A_2\bf{c}_1) &= -\det\begin{bmatrix} \bf{u}_i & \bf{u}_j & A_1\bf{c}_2 \end{bmatrix} \det\begin{bmatrix} \bf{v}_i & \bf{v}_j & A_2\bf{c}_1 \end{bmatrix}\\
   \label{eq:before det identity} &= -\det\begin{bmatrix} \frac{1}{w_{1i}} A_1 \bf{q}_i & \frac{1}{w_{1j}} A_1 \bf{q}_j & A_1\bf{c}_2 \end{bmatrix} \det\begin{bmatrix} \frac{1}{w_{2i}} A_2\bf{q}_i & \frac{1}{w_{2j}} A_2 \bf{q}_j & A_2\bf{c}_1 \end{bmatrix} \\
         \label{eq:after det identity}   &= -\det\begin{bmatrix} \frac{1}{w_{1i}} \bf{q}_i & \frac{1}{w_{1j}} \bf{q}_j & \bf{c}_2 & \bf{c}_1 \end{bmatrix} \det\begin{bmatrix} \frac{1}{w_{2i}} \bf{q}_i & \frac{1}{w_{2j}} \bf{q}_j & \bf{c}_1 & \bf{c}_2 \end{bmatrix} \\
            &= \frac{1}{w_{1i}w_{1j}w_{2i}w_{2j}} \det\begin{bmatrix}  \bf{q}_i & \bf{q}_j & \bf{c}_1 & \bf{c}_2 \end{bmatrix} \det\begin{bmatrix}  \bf{q}_i & \bf{q}_j & \bf{c}_1 & \bf{c}_2 \end{bmatrix} \\
            &= \frac{1}{w_{1i}w_{1j}w_{2i}w_{2j}} (\det\begin{bmatrix}  \bf{q}_i & \bf{q}_j & \bf{c}_1 & \bf{c}_2 \end{bmatrix})^2.
\end{align}
\Cref{eq:after det identity} follows from \Cref{eq:before det identity} by applying \Cref{lem:3det to 4det} to both determinants in the product. Since \\ $D_{ij}(-A_1\bf{c}_2 , A_2\bf{c}_1)  \neq 0$ by assumption, we conclude that $\sign D_{ij}(-A_1\bf{c}_2 , A_2\bf{c}_1) = \sign (w_{1i}w_{2i})(w_{1j}w_{2j})$.
\end{proof}




\begin{lemma}
\label{lem:det det infer gigj}
Let $X$ be a fundamental matrix of $\mathcal{P}$. Suppose 
$ D_{ij}( \adj(X)\bf{t} , \bf{t}) \neq 0$ where $\bf{t}^\top X = \bf{0}$. Then \\ $\sign  D_{ij}( \adj(X)\bf{t} , \bf{t}) = \sign g_ig_j(X)$. 
\end{lemma}
\begin{proof}
Write $X = [\bf{t}]_\times G$ for $\bf{t} \in \R^3 \setminus\{\bf{0}\}$ and some $G \in \GL_3$. We know that $\bf{t}$ and $\adj(X)\bf{t}$ generate the one-dimensional left and right kernels of $X$, respectively. Since
$D_{ij}( \adj(X)\bf{t}, \bf{t}) \neq 0$,  we know that $\mathbf{t} \neq 0, \adj(X) \mathbf{t} \neq 0$, $\adj(X)\bf{t}$ is not collinear with $\bf{u}_i$ and $\bf{u}_j$, and $\bf{t}$ is not collinear with $\bf{v}_i$ and $\bf{v}_j$. In particular, this means that neither $\bf{u}_i$ nor $\bf{u}_j$ is the right kernel of $X$ and neither $\bf{v}_i$ nor $\bf{v}_j$ is the left kernel of $X$. It follows that $X$ is $(\bf{u}_i, \bf{v}_i)$ regular and $(\bf{u}_j, \bf{v}_j)$ regular and neither is the epipole pair. By \Cref{thm:equivalentprconditions}, there exists a finite projective reconstruction $(A_1 = \begin{bmatrix} I & \bf{0}\end{bmatrix}, A_2 = \begin{bmatrix} G & \bf{t}\end{bmatrix} , \{\bf{q}_i, \bf{q}_j\} )$ of $\{(\bf{u}_i, \bf{v}_i), (\bf{u}_j, \bf{v}_j)\}$  such that the world points $\bf{q}_i, \bf{q}_j$ are not on the baseline. 

\Cref{lem:wi to triple products} implies that $\sign D_{ij}( -A_1\bf{c}_2 , A_2\bf{c}_1) =  \sign (w_{1i}w_{1j})(w_{2i}w_{2j})$. By \Cref{thm:equivalentchiralconditions}, $(w_{1i}w_{1j})(w_{2i}w_{2j}) > 0 $ if and only if $\left[(\bf{t} \times \bf{v}_i)^\top(\bf{t} \times G\bf{u}_i)\right]\left[(\bf{t} \times \bf{v}_j)^\top(\bf{t} \times G\bf{u}_j)\right] > 0$. 
Combining this fact with \Cref{lem:gigj is quad product}, it follows that $\sign D_{ij}( -A_1\bf{c}_2 , A_2\bf{c}_1 ) = \sign g_ig_j(X)$. 
Finally note that $A_2 \bf{c}_1 = \bf{t}$ and $-A_1 \bf{c}_2 = \det(G) (G^{-1}\bf{t}) $ which is a positive multiple of $\adj(X)\bf{t}$. Indeed,
\begin{align} \label{eq:adj(X)t not 0}
\adj(X)\bf{t} = \adj([\bf{t}]_\times G)\bf{t}=  \adj(G)\adj([\bf{t}]_\times) \bf{t} = \det(G)G^{-1}(\bf{t}\bf{t}^\top)\bf{t} = \|\bf{t}\|^2 \det(G) (G^{-1}\bf{t}).
\end{align}
Substituting $\adj(X)\bf{t}$ for $-A_1 \bf{c}_2$ and $\bf{t}$ for $A_2 \bf{c}_1$, the result follows. 
\end{proof}

The computation in the previous proof, in particular \eqref{eq:adj(X)t not 0}, shows that if $\mathbf{t}$ is a non-zero generator of the 
left kernel of a fundamental matrix $X$ then $\adj(X) \mathbf{t}$ is a non-zero generator of the right 
kernel of $X$.

Note that $\det \begin{bmatrix} \bf{u}_i & \bf{u}_j & A_1 \bf{c}_2 \end{bmatrix}$ can be zero without $\bf{u}_i$ or $\bf{u}_j$ being the epipole $A_1 \bf{c}_2$. Indeed, by \Cref{lem:3det to 4det} this happens whenever $\bf{q}_i , \bf{q}_j, \bf{c}_1 , \bf{c}_2$ are coplanar. On the other hand, \Cref{lem:chiral zeros} implies that $g_i$ vanishes at $X$ if and only if $\bf{u}_i$ or $\bf{v}_i$ is an epipole of $X$. Therefore, $D_{ij}(\adj(X)\bf{t}, \bf{t})$ may vanish even when $g_ig_j(X) \neq 0$.

\Cref{lem:det det infer gigj} shows that knowing the specific generators of the kernels of $X$, i.e., $\bf{t}$ and $\adj(X)\bf{t}$, respectively, is enough to compute the sign of the chiral epipolar inequalities.
Note that a choice of generator $\bf{t}$ for the left kernel of $X$ determines a signed generator $\adj(X)\bf{t}$ for the right kernel.
When $D_{ij}(\adj(X)\bf{t}, \bf{t})$ does not vanish, we can use it to infer the validity of chiral epipolar inequalities via \Cref{lem:det det infer gigj}, and hence argue for the existence of a chiral reconstruction of $\mathcal{P}$.
We now identify a situation where we can use {\em any} generators of the kernels of $X$ in $D_{ij}$.

\begin{definition}
\label{def:I(X)}
Suppose $X$ is a fundamental matrix of $\PP$. Define $I(X)$ to be the set of indices $i$ such that $g_i(X) \neq 0$, i.e., the index set of inactive chiral polynomials at $X$. Let
$\mathcal{P}_{I(X)}$ be the subset of point pairs in $\mathcal{P}$ indexed by 
$I(X)$.
\end{definition}

\begin{theorem}
\label{thm:chiral inactive indices}
Let $X$ be a fundamental matrix of $\mathcal{P}$ where $\bf{e}_1$ and $\bf{e}_2$ generate the right and left kernels of $X$, respectively. Suppose 
$|I(X)| \ge 3$, and $D_{ij}(\bf{e}_1 , \bf{e}_2) \neq 0$ for all $i,j \in I(X)$. Then there exists a chiral reconstruction of $\mathcal{P}_{I(X)}$ associated to $X$ if and only if $D_{ij}( \bf{e}_1 , \bf{e}_2)$ has the same sign for all $i,j \in I(X)$.
\end{theorem} 

\begin{proof} 
 Suppose there exists a chiral reconstruction of $\mathcal{P}_{I(X)}$ associated to $X$. Then by \Cref{thm:chiral F existence gigj}, 
 $g_ig_j(X) \geq 0$ for all $i,j \in I(X)$. In fact, $g_ig_j(X) > 0$ for all $i,j \in I(X)$ since if $g_ig_j(X) = 0$ for some $i,j$ while 
 $D_{ij}(\adj(X) \bf{t} , \bf{t} ) \neq 0$, we would contradict \Cref{lem:det det infer gigj}. Indeed, if $D_{ij}(\bf{e}_1 , \bf{e}_2) \neq 0$ for some 
 kernel generators $\bf{e}_1, \bf{e}_2$, it remains non-zero for any other pair of kernel generators. By \Cref{thm:chiral F existence gigj}, $D_{ij}(\adj(X) \bf{t} , \bf{t} ) $ has the same sign for all $i,j$, and since 
 $\adj(X) \bf{t}$ and $\bf{t}$ are (non-zero) generators of the right and left kernels of $X$, the result follows.

Conversely, suppose $D_{ij}( \bf{e}_1, \bf{e}_2 )$ has the same non-zero sign for all $i,j \in I(X)$ where $\bf{e}_1$ and $\bf{e}_2$ generate the right and left kernels of $X$, respectively. 
Then $\bf{e}_1 = \lambda \adj(X) \bf{e}_2$ for some non-zero $\lambda$ by \eqref{eq:adj(X)t not 0}. 
By \Cref{lem:det det infer gigj}, we know 
$$\sign D_{ij}( \bf{e}_1 , \bf{e}_2) = \lambda  \sign D_{ij}( \adj(X)  \bf{e}_2 , \bf{e}_2) = \lambda \sign g_ig_j(X)$$
 for all $i,j$. This shows that $ g_ig_j(X)$ has the same sign for all $i,j \in I(X)$. Since $|I(X)| \ge 3$, this common sign cannot be negative and hence 
 $g_ig_j(X) > 0$ for all $i,j \in I(X)$. These strict inequalities also imply that $X$ is $\mathcal{P}_{I(X)}$-regular. 
 Then by  \Cref{thm:chiral F existence gigj} there is a chiral reconstruction of $\mathcal{P}_{I(X)}$ associated to $X$. 
\end{proof}

We remark that $D_{ij}(\bf{e}_1, \bf{e}_2)$ does not have a well-defined sign on $\P_\R^2\times \P_\R^2$ because it is linear in $\bf{e}_1$ and $\bf{e}_2$. To get an inequality description of chirality in epipole space, we can take pairwise products $D_{ij}(\bf{e}_1,\bf{e}_2) D_{ik}(\bf{e}_1, \bf{e}_2)$ which are quadratic in each $\P_\R^2$ factor. If $D_{ij}(\bf{e}_1,\bf{e}_2) D_{ik}(\bf{e}_1, \bf{e}_2) > 0$, then $g_jg_k(X)>0$ for any fundamental matrix $X$ with epipoles $\bf{e}_1$ and $\bf{e}_2$. However, since $D_{ij}(\bf{e}_1, \bf{e}_2)$ may vanish even when $g_ig_j(X)$ does not, we observe that $D_{ij}(\bf{e}_1,\bf{e}_2) D_{ik}(\bf{e}_1, \bf{e}_2) \ge 0$ for all triples $i,j,k$ is not equivalent to $g_ig_j(X) \ge 0, g_ig_k(X) \ge 0 $ and $g_ig_k(X) \ge 0$. Due to this subtlety, we primarily study chirality using $g_ig_j \ge 0 $ in $\P^8_\R$ as opposed to $D_{ij}D_{ik} \ge 0$ in $\P^2_\R \times \P^2_\R$. 

\subsection{Three point pairs always have a chiral reconstruction}
\label{subsec:three pairs}
In this section, we apply the tools developed so far to show that there is always a chiral reconstruction when $|\mathcal{P}| = 3$, and hence also when $|\mathcal{P}| \leq 3$ since $\mathcal{P}$ can have a chiral reconstruction only if all its 
subsets have one. We begin with two technical lemmas.

\begin{lemma} \label{lem:match chirotope}
Suppose $\bf{a}_1, \bf{a}_2, \bf{a}_3$ are three non-collinear points in 
$\R^3$. Then for each of the eight elements in $\sigma \in \{+,-\}^3$, there is an $\bf{e} \in \R^3$ such that 
$\bf{a}_1, \bf{a}_2, \bf{a}_3, \bf{e}$ are in general position (no three in a line) and 
$$\sigma = (\sign(\det[\bf{a}_1 \, \bf{a}_2 \, \bf{e}]),\,
\sign(\det[\bf{a}_1 \, \bf{a}_3 \, \bf{e}]),\,
\sign(\det[\bf{a}_2 \, \bf{a}_3 \, \bf{e}])).$$
Further, for each $\sigma \in \{+,-\}^3$, the corresponding choices of $\bf{e}$ 
come from an open polyhedral cone in $\R^3$.
\end{lemma}

\begin{proof}

The expression $\det[\bf{a}_i \, \bf{a}_j \, \bf{e}] = l_{ij}(\bf{e})$ is 
the linear form whose kernel is the span of $\bf{a}_i$ and $\bf{a}_j$. 
Since $\bf{a}_1, \bf{a}_2, \bf{a}_3$ are non-collinear, the hyperplanes 
cut out by $l_{12}(\bf{e}), l_{13}(\bf{e}), l_{23}(\bf{e})$ 
divide $\R^3$ into eight regions, each of which is a polyhedral cone.
The interiors of these cones correspond to the eight sign patterns $\sigma$.
\end{proof}

For $\mathbf{v}_1,\mathbf{v}_2 \in \R^3$, let $\cone(\mathbf{v}_1 , \mathbf{v}_2) := \{\la_1\mathbf{v}_1 + \la_2 \mathbf{v}_2: \la_1,\la_2 \ge 0\}$ be the convex cone spanned by $\mathbf{v}_1$ and $\mathbf{v}_2$.

\begin{lemma}
\label{lem:t cross quad product}
Suppose $\bf{v}_l, \bf{v}_r, \bf{t}$ are points in $\R^3$ on an affine line $L$.
 If $\bf{t} \notin \cone(\bf{v}_l, \bf{v}_r) $, then for all $\bf{w}_1, \bf{w}_2 \in \cone(\bf{v}_l, \bf{v}_r)$, $(\bf{t}\times \bf{w}_1)^\top (\bf{t} \times \bf{w}_2) > 0$.
\end{lemma}

\begin{proof} 
If $\bf{t} \notin  \cone(\bf{v}_l, \bf{v}_r) $, then either $\bf{v}_l \in \cone(\bf{t}, \bf{v}_r)$ or $\bf{v}_r \in \cone(\bf{v}_l , \bf{t})$. Suppose $\bf{v}_l \in \cone(\bf{t}, \bf{v}_r)$. Since $\bf{w}_1, \bf{w}_2 \in \cone(\bf{v}_l, \bf{v}_r)$ and $\cone(\bf{v}_l, \bf{v}_r)\subseteq \cone(\bf{t}, \bf{v}_r)$, we know $\bf{w}_1, \bf{w}_2 \in \cone(\bf{t}, \bf{v}_r)$. Write $\bf{w}_1 = \la_1 \bf{t} + \la_2 \bf{v}_r $ and $\bf{w}_2 = \mu_1 \bf{t} + \mu_2 \bf{v}_r$ where $\lambda_i, \mu_j \geq 0$. Since 
$\bf{w}_i \neq \bf{t}$, $\lambda_2, \mu_2 > 0$. The result follows from direct computation using that $\bf{t} \neq \bf{v}_r$:
\begin{align}
    (\bf{t}\times \bf{w}_1)^\top (\bf{t} \times \bf{w}_2) &= (\bf{t}\times (\la_1 \bf{t} + \la_2 \bf{v}_r) )^\top (\bf{t} \times (\mu_1 \bf{t} + \mu_2 \bf{v}_r )) = \la_2\mu_2(\bf{t}\times\bf{v}_r)^\top (\bf{t} \times \bf{v}_r ) > 0.
\end{align}
Similar reasoning applies if $\bf{v}_r \in \cone(\bf{v}_l , \bf{t})$.
\end{proof}

\begin{theorem}
If $|\mathcal{P}|=3$ then $\mathcal{P}$ has a chiral reconstruction.
\end{theorem}

\begin{proof} We break the proof into two parts:
\begin{enumerate}
    \item Suppose $\mathcal{U} = \{\bf{u}_1, \bf{u}_2, \bf{u}_3 \}$ 
    or $\mathcal{V} = \{ \bf{v}_1, \bf{v}_2, \bf{v}_3 \}$ is in general position, say  $\mathcal{U}$ is non-collinear. 
    Choose $\bf{e}_2$ not on the line spanned by $\bf{v}_i$ and $\bf{v}_j$ for any $i,j$, so that $\det\begin{bmatrix} \bf{v}_i& \bf{v}_j & \bf{e}_2 \end{bmatrix} \neq 0$ for all $i,j$. By \Cref{lem:match chirotope}, there exists an $\bf{e}_1$ such that $D_{ij}(\bf{e}_1, \bf{e}_2)$ has the same non-zero sign for all $i,j$. Since $k=3$ and $\bf{e}_1$ and $\bf{e}_2$ are chosen from open regions, 
    $W_{\bf{e}_1} \cap W^{\bf{e}_2}$ contains at least one rank two matrix $X$. By construction, this $X$ is a $\mathcal{P}$-regular fundamental matrix with epipoles $\bf{e}_1, \bf{e}_2$ and $I(X) = \{1,2,3\}$. By \Cref{thm:chiral inactive indices}, $X$ yields a chiral reconstruction of $\mathcal{P}$.

    \item Suppose both $\mathcal{U}$ and $\mathcal{V}$ are collinear and consider the affine lines $L_{\mathcal{U}}$ and $L^{\mathcal{V}}$ in $\R^3$ spanned by these points, which all have last coordinate $1$. 
     Let $\bf{u}_l, \bf{u}_r$ be the furthest left and right points on the $L_\mathcal{U}$ line, so that the third point lies strictly between $\bf{u}_l, \bf{u}_r$. Similarly let $\bf{v}_l, \bf{v}_r$ be the furthest left and right points on the $L^\mathcal{V}$ line. Let $\bf{t} \in L^{\mathcal{V}} \setminus \cone(\bf{v}_l, \bf{v}_r)$ and choose $G \in \GL_3$ such that $G\bf{u}_l = \bf{v}_l$ and $G\bf{u}_r = \bf{v}_r$. Define $X = [\bf{t}]_\times G$. Since $\bf{t}, \bf{v}_i, G\bf{u}_i$ are collinear for all $i$, the $i$th epipolar equation is satisfied. Since the chosen epipoles for $X$ do not coincide with any data points, $X$ is a $\mathcal{P}$-regular fundamental matrix. By construction $G\bf{u}_i \in \cone(\bf{v}_l, \bf{v}_r)$ for each $i$. Combining \Cref{lem:gigj is quad product} and \Cref{lem:t cross quad product}, it follows that $g_i(X) > 0$ for each $i$, and there is a chiral reconstruction of $\mathcal{P}$ associated to $X$ by \Cref{thm:chiral F existence gigj}.

\end{enumerate}

\end{proof}

\subsection{Walls and Corners}
\label{subsec:chiral boundary}

To understand the existence of chiral reconstructions when $|\mathcal{P}| \geq 4$, we need one more tool that we now develop. Recall that the 
chiral epipolar region of $\mathcal{P}$ is the set of $\mathcal{P}$-regular 
fundamental matrices that live in the semialgebraic subset of the 
real epipolar variety cut out by the chiral epipolar inequalities. 
\Cref{lem:chiral zeros} implies that the chiral epipolar region is bounded by the $W_{\bf{u}_i}$, $W^{\bf{v}_j}$ walls. The fundamental matrices on walls are generally $\mathcal{P}$-irregular and do not correspond to a reconstruction. However, we show that $\mathcal{P}$-irregular fundamental matrices that are smooth points of the epipolar variety and 
yield partial chiral reconstructions, can be perturbed to  $\mathcal{P}$-regular fundamental matrices that yield chiral reconstructions of $\mathcal{P}$.

\begin{lemma}
\label{lem:tangent direction}
Suppose $\RR_2 \cap \mathcal{L}_\PP$ is irreducible. If $X$ is a smooth fundamental matrix that is $(\bf{u}_i,\bf{v}_i)$-irregular, then there is a tangent direction $\bf{d} \in T_X(\RR_2 \cap \mathcal{L}_\mathcal{P})$ such that the directional derivative $D_\bf{d}g_i(X) \neq 0$. 
\end{lemma}

\begin{proof}
Suppose $X$ is a smooth fundamental matrix of $\mathcal{P}$. Smoothness implies that the tangent space at $X$ to the epipolar variety has the same dimension as the variety. If $X$ is $(\bf{u}_i, \bf{v}_i)$-irregular for some $i$, then $X$ is in exactly one of $W_{\bf{u}_i}$ or $W^{\bf{v}_i}$. Since $\RR_2 \cap \mathcal{L}_\PP$ is irreducible, each wall must be an embedded component of strictly smaller dimension. This means that the wall's tangent space is strictly contained in the tangent space of the epipolar variety at $X$. Therefore, we can choose a direction $\bf{d}$ tangent to the epipolar variety at $X$ which is not tangent to the wall which contains $X$. \Cref{lem:chiral zeros} implies that $g_i$ vanishes on the real part of the epipolar variety only on the walls. By construction $D_\bf{d}g_i(X) \neq 0$.
\end{proof}

The following lemma is needed for \Cref{thm:walk away from walls} below, but its proof might be best understood after \Cref{sec:classicalAG}. 

\begin{lemma}\label{lem:smoothwalls}
    Suppose $|\cP| \leq 5$ and $\RR_2\cap \LL_\PP$ is irreducible. If a wall $W_{\vu_i}$ (or $W^{\bf{v}_i}$) contains a matrix of rank two, then a generic point $Y$ on the wall is a smooth point of $\RR_2\cap \LL_\PP$. 
\end{lemma}

\begin{proof}

    We reduce to the case of $\dim(\LL_\PP) = 3$ as follows.
    If the wall contains a smooth point, then so will its intersection with generic data planes $L_{(\bf{u}_j, \bf{v}_j)}$. Therefore, cutting with sufficiently many of these, using Bertini's Theorem, we can assume that $\LL_\PP$ has dimension three, $\RR_2\cap \LL_\PP$ is an irreducible cubic surface in $\P^3$, and $W_{\vu_i}$ is a line on it. Suppose for contradiction that $\RR_2\cap \LL_\PP$ is singular at every point in $W_{\vu_i}$.
    
    The cubic surfaces which are singular along a line have been classified, see e.g.~\cite[in particular Case E]{singcubic}. We show that $\RR_2\cap \LL_\PP$ cannot be any of these types, essentially because it contains too many intersecting lines. Indeed, $W_{\vu_l}$ intersects $W^{\vv_m}$ as long as $l\neq m$ because the equations $X\bf{u}_l = \bf{0}$ and $\bf{v}_m^\top X = \bf{0}$ impose at most three additional conditions on the three dimensional $\LL_\PP$. Additionally, the assumption that the wall $W_{\vu_i}$ contains a matrix of rank two implies that this wall does not coincide with $W_{\vu_j}$ for $j\neq i$. 
    
    The first examples of cubic surfaces singular along a line are the cones over a singular plane cubic curve. Our epipolar variety cannot be such a surface because it contains intersecting lines with distinct intersection points, which these cones do not. There are only two other types of cubic surfaces (up to change of coordinates) that are singular along a line.
    
    The next type that is singular along a line, contains a one-dimensional family of lines, and one more line. A representative is given by the equation $w^2y + x^2z$, which contains the lines $\mathcal{V}(w,x)$, $\mathcal{V}(y,z)$ and a family of lines that form a twisted cubic in the Pl\"{u}cker quadric ${\rm Gr}(2,4)$ of lines in $\P^3$. The lines in the family are pairwise skew. This type of singular cubic surface only contains two lines that intersect lines in the family, which is inconsistent with the intersection pattern of lines on $\RR_2 \cap \LL_\PP$. The last relevant type is represented by the equation $w^2y + wxz + x^3$. It is singular along the line $\mathcal{V}(x,w)$ and it contains a one-dimensional family of lines. This family of lines is a twisted cubic curve in ${\rm Gr}(2,4)$ and they are mutually skew, which also does not fit the intersection pattern on our epipolar variety. 
  
    This discussion of cases shows that the epipolar variety cannot be singular along an entire line if it is an irreducible cubic surface in $\P^3$.
\end{proof}

Recall $I(X)$ and $\mathcal{P}_{I(X)}$ from \Cref{def:I(X)}.
\begin{theorem}
\label{thm:walk away from walls}
Suppose $|\cP| \leq 5$ and $\mathcal{R}_2\cap \mathcal{L}_P$ is irreducible.
Then there exists a chiral reconstruction of $\mathcal{P}$ if and only if $\mathcal{P}_{I(X)}$ has a chiral reconstruction associated to some smooth $X \in \RR_2 \cap \cL_\mathcal{P}$.
\end{theorem}

\begin{proof}
Suppose there exists a chiral reconstruction of $\mathcal{P}$. By \Cref{thm:chiral F existence gigj}, there exists a fundamental matrix $X$, which is $(\bf{u}_i, \bf{v}_i)$-regular for all $i$ and $g_ig_j(X) \ge 0$ for all $1\le i<j \le k$. If $g_ig_j(X) > 0$ for all $i,j$, then the semialgebraic subset of the epipolar variety described by the chiral epipolar inequalities has non-empty interior. Since every non-empty, open semialgebraic set in an algebraic variety intersects its smooth locus, there is a smooth $X' \in \RR_2\cap \cL_\cP$ associated to a chiral reconstruction. Otherwise, $\PP$-regularity implies that at most one $(\bf{u}_i, \bf{v}_i)$ can be the epipole pair of $X$, i.e., at most one $g_i$ can vanish at $X$. Without loss of generality, suppose $g_1(X) = 0$ so that $X \in W_{\bf{u}_1} \cap W^{\bf{v}_1}$. 
By \Cref{lem:smoothwalls}, a generic point in $W_{\vu_1}$ is a smooth point of $\RR_2\cap \LL_\PP$. We can move away from the corner $W_{\vu_i}\cap W^{\vv_i}$ along the wall to a smooth point of $\RR_2\cap \LL_\PP$ without changing the signs of any of the chiral epipolar inequalities, but keeping $g_1=0$. This gives a chiral reconstruction of $\PP_{I(X)}$ associated to a smooth fundamental matrix, finishing the proof of the first implication. 

Conversely, suppose there is a chiral reconstruction of $\cP_{I(X)}$ associated to a smooth $X\in\RR_2 \cap \cL_\cP$. If $\cP = \cP_{I(X)}$, there is nothing to show. Otherwise, we have $g_i(X) = 0$ for some $i$, which means that $X$ lies on some walls by \Cref{lem:chiral zeros}. Smoothness of $X$ implies that $X$ must have rank two, so its left and right kernels are one-dimensional. Since the 
$\bf{u}_i$'s and $\bf{v}_j$'s are distinct, this means $X$ may lie on at most one wall $W_{\bf{u}_i}$ and at most one wall $W^{\bf{v}_j}$. In other words, 
$I(X)$ must be one of $[k]\setminus\{i\}$ or $[k]\setminus\{i,j\}$.

Suppose $I(X) = [k]\setminus\{i\}$. We argue locally: Choose an affine chart around $X$ such that $g_j(X)>0$ for all $j\in [k]\setminus\{i\}$, which is possible because we know that $g_jg_\ell(X) > 0$ for all pairs $j,\ell\in [k]\setminus\{i\} = I(X)$. If $X in W_{\mathbf{u}_i} \cap W^{\mathbf{v}_i}$, it is $\mathcal{P}$-regular and is associated to a chiral reconstruction. Otherwise,  \Cref{lem:tangent direction} implies that there is a tangent direction $\bf{d} \in T_X (\RR_2\cap \cL_\cP)$ such that $D_\bf{d} g_i(X) \neq 0$. Then we have $D_\bf{d} (g_ig_j)(X) = (D_\bf{d}g_i)(X) g_j(X)$ because $g_i(X) = 0$. Since $g_j(X) > 0$ in our affine chart, the sign of the directional derivative of $g_ig_j$ is determined by $(D_\bf{d} g_i)(X)$, which is independent of $j$. By flipping the sign of $\bf{d}$, if necessary, we can assume that $(D_\bf{d} g_i)(X)>0$. By Taylor approximation, there is a nearby smooth point $X'\in \RR_2\cap \cL_\cP$ such that $g_jg_\ell(X) > 0$ for all $j,\ell \in [k]$. By \Cref{thm:chiral F existence gigj}, there is a chiral reconstruction of $\cP$ corresponding to $X'$.

Suppose $I(X) = [k]\setminus\{i,j\}$. Choose an affine chart around $X$ such that $g_l(X)>0$ for all $l\in [k]\setminus\{i,j\}$, which is possible because we know that $g_lg_m(X) > 0$ for all pairs $l,m \in [k]\setminus\{i,j\} = {I(X)}$. Choose linearly independent tangent vectors $\bf{d}_i, \bf{d}_j \in T_X (\RR_2\cap \cL_\cP)$ such that the directional derivatives $D_{\bf{d}_i} g_i(X)$ and $D_{\bf{d}_j} g_j(X)$ are non-zero. Choose a linear combination $\bf{d} \in \Span\{\bf{d}_i, \bf{d}_j\}$ such that $(D_{\bf{d}} g_i)(X)>0$ and $(D_{\bf{d}}g_j)(X)>0$. Since $g_l(X) > 0$ in our affine chart, the sign of the directional derivatives $(D_\bf{d} g_ig_l)(X)$ and $(D_\bf{d} g_jg_l)(X)$ at $X$ are determined by $(D_\bf{d} g_i)(X)$ and $(D_\bf{d} g_j)(X)$, respectively. By Taylor approximation, there is a nearby smooth point $X'\in \RR_2\cap \cL_\cP$ such that $g_lg_m(X) > 0$ for all $l,m \in [k]$. By \Cref{thm:chiral F existence gigj}, there is a chiral reconstruction of $\cP$ associated to $X'$.
\end{proof}

To apply \Cref{thm:walk away from walls}, it is useful to understand the 
smooth locus of the epipolar variety which we describe next. 
For a more general discussion of tangent spaces to rank varieties we refer to \cite[Example~14.16]{Harris95}.
\begin{lemma}
\label{lem:smooth points}
Suppose $X \in \RR_2 \cap \mathcal{L}_\mathcal{P}$ is a rank two matrtix such that $\bf{u}$ spans the right kernel of $X$ and $\bf{v}$ spans the left kernel of $X$. Then $X$ is a smooth point of $\RR_2 \cap \mathcal{L}_\mathcal{P}$ if and only if $\vv\vu^\top$ does not lie in the span of $\{\vv_i\vu_i^\top\colon i=1,\ldots,k\}\subset \P^8$.
\end{lemma}
\begin{proof}
The gradient of the determinant of $X$ is the cofactor matrix of $X$, which is $\adj(X)^\top$. Since $\bf{u}\bf{v}^\top$ and $\adj(X)$ are collinear, the tangent hyperplane to $\RR_2$ at $X$ is 
\[
T_X\RR_2 = \{M\in \P^8\colon \langle M, \vv\vu^\top \rangle = 0\}.
\]
Since it is a hyperplane, it intersects $\cL_\cP$ transversely if and only if $\cL_\cP\not\subset T_X\RR_2$, or equivalently, 
if and only if $\vv\vu^\top$ does not lie in the span of $\{\vv_i\vu_i^\top\colon i=1,\ldots,k\}\subset \P^8$.
\end{proof}

\section{Four Point Pairs}
\label{sec:fourpoints}

In the previous section we saw that any set of three point pairs will always have a chiral reconstruction. In this section we completely analyze the case of four point pairs. In \Cref{thm:4 chiral exists} we prove that if the point configurations in the two views have the same {\em rank}, then they admit a chiral reconstruction. Otherwise, a chiral reconstruction can fail to exist (\Cref{thm:4 no chiral}). Our main tool will be \Cref{thm:walk away from walls} which requires understanding when the epipolar variety is reducible. We do this in \Cref{subsec:irreducibility} which leads to our chirality results in \Cref{subsec:chiral reconstructions k=4}. 

\subsection{Irreducibility of the epipolar variety} 
\label{subsec:irreducibility}
Suppose the epipolar variety $\RR_2 \cap \mathcal{L}_\mathcal{P}$ is reducible. 
Then $\mathcal{L}_\mathcal{P}$ is not contained in $\RR_2$ since otherwise, $\RR_2 \cap \mathcal{L}_\mathcal{P} = \mathcal{L}_\mathcal{P}$ which is irreducible.
Therefore, $\RR_2 \cap \mathcal{L}_\mathcal{P}$ is a proper subvariety of 
$\mathcal{L}_\mathcal{P}$. Further, any irreducible component of $\RR_2 \cap 
\mathcal{L}_\mathcal{P}$ has codimension one in $\mathcal{L}_\mathcal{P}$ because
$\det(X)=0$ is the only additional condition to those defining $\mathcal{L}_\mathcal{P}$. Conversely, if $C$ is a proper subvariety of $\RR_2 \cap \mathcal{L}_\mathcal{P}$ of codimension one in $\mathcal{L}_\mathcal{P}$, 
then $C$ is an irreducible component of 
$\RR_2 \cap \mathcal{L}_\mathcal{P}$. Indeed, if $C$ lies in some irreducible component of $\RR_2 \cap \mathcal{L}_\mathcal{P}$ of dimension larger than $\dim(C)$, then $C$ has codimension at least two in $\mathcal{L}_\mathcal{P}$, which is a contradiction.
Lastly, since $\RR_2 \cap \mathcal{L}_\mathcal{P}$ has degree three, it must have a linear component if it is reducible and this linear component is a subspace of $\RR_2$. Thus to understand whether $\RR_2 \cap \mathcal{L}_\mathcal{P}$ is reducible, we need to understand the 
linear subspaces of $\RR_2$ that have codimension one in $\mathcal{L}_\mathcal{P}$. 

A subspace $C \subset \C^{3 \times 3}$ is said to have rank at most $k$ if $k$ is the maximum rank of matrices in $C$. 
The following definition is from \cite{EisenbudHarris} for $3\times 3$ matrices.
\begin{definition}
Let $C \subset \C^{3 \times 3}$ be a subspace of rank at most $k$. Then $C$ is called a {\em compression space} if there exists a subspace $V \subseteq \C^3$ of codimension $k_1$ and a subspace $W \subseteq \C^3$ of dimension $k_2$ such that $k_1 + k_2 = k$, and every $X \in C$ maps $V$ into $W$ ({\em compresses $V$ into $W$}). We refer to $C$ as a $(k_1,k_2)$-compression space.
\end{definition}

Note that every subspace of a $(k_1,k_2)$-compression space is again a $(k_1,k_2)$-compression space. The following theorem from 
\cite{EisenbudHarris} attributed to Atkinson \cite{Atkinson} tells us what subspaces of $\RR_2$ look like.

\begin{theorem}\cite[Theorem 1.1]{EisenbudHarris}  \label{thm:EisenbudHarris}
A vector space of matrices in $\C^{3 \times 3}$ 
of rank at most two is either a compression space or a subspace of the linear space of $3 \times 3$ skew-symmetric matrices.
\end{theorem}

\begin{theorem} 
\label{thm:codim 1 compression}
Suppose $|\mathcal{P}| = 4$. Then the epipolar variety $\RR_2 \cap \mathcal{L}_\mathcal{P}$ is reducible if and only if $\mathcal{L}_\mathcal{P}$ contains a compression space of rank at most two and codimension one.
\end{theorem}

\begin{proof}
By the above discussion, if $\RR_2 \cap \mathcal{L}_\mathcal{P}$ is reducible, then it has a 
linear component $C$ that is a subspace of $\RR_2$ and has codimension one in $\mathcal{L}_\mathcal{P}$. Since $\dim(\mathcal{L}_\mathcal{P}) \geq 4$ when $|\mathcal{P}| = 4$, and the space of $3 \times 3$ skew-symmetric matrices 
has dimension two, by \Cref{thm:EisenbudHarris},
$C$ must be a compression space of rank at most two. 
Conversely, if $C$ is a compression space in $\mathcal{L}_\mathcal{P}$ of rank at most two and codimension one, then  
$C \subseteq \RR_2 \cap \mathcal{L}_\mathcal{P}$, and since $C$ has codimension 
one in $\mathcal{L}_\mathcal{P}$, $C$ must be an irreducible component of $\RR_2 \cap \mathcal{L}_\mathcal{P}$.
\end{proof}

\begin{example}  \label{ex:compression space ex1}
Consider the three point pairs given by the columns of the following matrices:

$$U = \begin{pmatrix}1&2&3\\
                    0&0&0\\
                    1&1&1\\
      \end{pmatrix}, \,\,\,\,\, V = \begin{pmatrix}0&0&0\\
                    1&2&3\\
                    1&1&1\\
      \end{pmatrix}.$$
      Using Macaulay2 one can compute that $\mathcal{R}_2 \cap \mathcal{L}_{\mathcal{P}} = Q \cup C$ where $Q$ has degree two and $C$ has degree one. The ideal of $C$ is $\langle x_{21}, x_{23}, x_{31},x_{33}\rangle $, so every $X \in C$ has the form
        \[ 
            \begin{pmatrix} x_{11} & x_{12} & x_{13} \\ 0 & x_{22} & 0 \\ 
            0 & x_{32} & 0 \end{pmatrix}.
      \]
      Since $\dim(C) = 4$ and $\dim(\mathcal{L}_\mathcal{P}) = 5$, 
      $C$ has codimension one in $\mathcal{L}_\mathcal{P}$.
      The linear space $C$ is a $(1,1)$-compression space of rank at most two; 
      each $X \in C$ compresses the line spanned by the columns of $U$ 
      into the orthogonal complement of the line spanned by the columns of $V$. 

Suppose we now enlarge the set of point pairs to 
$$U = \begin{pmatrix}1&
      2&
      3 & 5 
      \\
      0&
      0&
      0 & 1 
    \\
      1&
      1&
      1 & 1\\
      \end{pmatrix}, \,\,\,\,\, V = \begin{pmatrix}0&
      0&
      0 & 1\\
      1&
      2&
      3 & 7 \\
      1&
      1&
      1 & 1\\
      \end{pmatrix}$$
      Now $\RR_2 \cap \mathcal{L}_\mathcal{P}$ has a linear component $C'$ consisting of matrices of the form 
      $$ \begin{pmatrix} x_{11} & x_{12} & x_{13} \\ 0 & x_{22} & 0 \\ 
            0 & x_{32} & 0 \end{pmatrix} \,\,\,\textup{ such that } \,\,\, 5x_{11} + x_{12} + x_{13} + 7x_{22} + x_{32}=0.$$
       This subspace $C'$ is clearly a subspace of the compression space $C$ and has codimension one in the new $\mathcal{L}_\mathcal{P}$.
       The linear condition arises from the epipolar equation $\langle X, \bf{v}_4 \bf{u}_4^\top \rangle = 0$ imposed by the new point pair. 
\end{example}

The main result of this subsection is the following 
which will allow us to use \Cref{thm:walk away from walls} to establish chirality results for four point pairs in the next subsection. A set of points in $\R^2$ is said to be in {\em general (linear) position} if no three of them are collinear.

\begin{theorem}
\label{thm:triples general implies irreducible}
Suppose $|\mathcal{P}| =  4$. If for all triples of distinct indices 
$i,j,k$, $\bf{u}_i,\bf{u}_j,\bf{u}_k$ 
or $\bf{v}_i,\bf{v}_j,\bf{v}_k$ are in general position, then $\RR_2 \cap \mathcal{L}_\mathcal{P}$ is irreducible. 
\end{theorem}

\Cref{thm:triples general implies irreducible} will follow from 
\Cref{thm:codim 1 compression} if under the assumptions of the theorem, there are no compression spaces of rank at most two and codimension one in $\mathcal{L}_\mathcal{P}$. We will now prove that this is indeed the case by understanding the orthogonal complements of compression spaces of rank at most two.

 Two vector spaces $C_1$ and $C_2$ in $\C^{3 \times 3}$ are {\em equivalent} if they have the same linear transformations up to a change of bases in the domain $\C^3$ and codomain $\C^3$, i.e., if there exists $A,B \in \textup{GL}_3$ such that $C_2 = \{ BXA \,:\, X \in C_1\}$. The following classification is straightforward.

\begin{lemma} \label{lem:compression spaces std form}
Let $C$ be a $(k_1,k_2)$-compression space of rank at most two in $\C^{3 \times 3}$.
\begin{enumerate}
    \item If $(k_1,k_2)=(2,0)$ then $C$ is equivalent to 
    a vector space of matrices with a zero last column.
    \item If $(k_1,k_2)=(0,2)$  then $C$ is equivalent to 
    a vector space of matrices with a zero last row.
    \item If $(k_1,k_2)=(1,1)$ then $C$ is equivalent to a vector space of matrices of the form 
    $$ \begin{pmatrix} 0 & 0 & \ast \\ 0 & 0 & \ast \\ \ast & \ast & \ast\end{pmatrix}.$$ 
\end{enumerate}
\end{lemma}




Call a $(k_1,k_2)$-compression space {\em maximal} if it is not a proper subspace of another 
$(k_1,k_2)$-compression space. In particular, 
a $(k_1,k_2)$-compression space is maximal if the equations cutting it out are precisely the conditions that impose the zeros guaranteed in their standard forms.
Non-maximal compression spaces satisfy further conditions on their potentially non-zero coordinates.
The following lemmas help us understand maximal compression spaces.

\begin{lemma} \label{lem:compression spaces and spans}
Suppose $\bf{u}, \bf{v} \in \C^3 \setminus \{0\}$. Then  $\bf{v}\bf{u}^\top \in S := \Span\{\bf{b}_1\bf{a}_1^\top, \bf{b}_1\bf{a}_2^\top,\dots, \bf{b}_n\bf{a}_{m-1}^\top,\bf{b}_{n} \bf{a}_{m}^\top\}$ if and only if $\bf{u} \in \Span \{\va_i\}_{i=1}^{m}$ and $\bf{v} \in \Span\{ \vb_i \}_{i=1}^{n}$.
\end{lemma}

\begin{proof}
Suppose $\bf{v}\bf{u}^\top \in S$. Then multiplying $\bf{v}\bf{u}^\top$ on the right by $\bf{u}$, we get that $\bf{v} \in \Span\{\bf{b}_1, \dots, \bf{b}_n\}$. 
Similarly, 
multiplying $\bf{v}\bf{u}^\top$ on the left by $\bf{v}^\top$,  we get that $\bf{u} \in \Span\{\bf{a}_1, \dots, \bf{a}_m\}$. Conversely, 
if $\bf{u} = \la_1 \bf{a}_1 + \dots, \la_m \bf{a}_m$ and $\bf{v} = \mu_1 \bf{v}_1 +  \dots, \mu_n \bf{v}_n$, then it is immediate that $\bf{v}\bf{u}^\top \in S$. Note that for this last statement to hold, it was important that all possible outerproducts of the form $\vb_i \va_j^\top$ are present in $S$. 
\end{proof}


\begin{lemma} \label{lem:orthogonal complements of compression spaces}
Let $C \subseteq \C^{3 \times 3}$ be a maximal $(k_1,k_2)$-compression space of rank at most two and let $C^\perp \subseteq \C^{3 \times 3}$ be its orthogonal complement. 
\begin{enumerate}
    \item If $(k_1,k_2)=(2,0)$, then $C^{\perp} = \textup{Span}\{\vb_1 \va^\top, 
    \vb_2 \va^\top, \vb_3\va^\top\}$ where $\va \in \C^3 \setminus \{0\}$ and $\vb_1, \vb_2, \vb_3$ are linearly independent vectors in $\C^3$. 
     Similarly if $(k_1,k_2) = (0,2)$, then $C^{\perp} = \textup{Span}\{\vb \va_1^\top, 
    \vb \va_2^\top, \vb \va_3^\top\}$ where $\vb \in \C^3 \setminus \{0\}$ and $\va_1, \va_2, \va_3$ are linearly independent vectors in $\C^3$. In both cases, as a projective space,  
    $\dim(C^\perp) = 2$ and all matrices in $C^\perp$ have 
    rank one. 
    \item If $(k_1,k_2)=(1,1)$, then $C^\perp = \textup{Span}\{\vb_1\va_1^\top$, $\vb_1\va_2^\top$, $\vb_2 \va_1^\top$, $\vb_2\va_2^\top \}$ where 
    $\vb_1, \vb_2$ are linearly independent and $\va_1, \va_2$ are linearly independent vectors in $\C^3$. As a 
    projective space, $\dim(C^\perp) = 3$ and the rank one matrices in it cut out a variety isomorphic to $\P^1 \times \P^1$.
\end{enumerate}
\end{lemma}

\begin{proof}
\begin{enumerate}
    \item
    Let $\overline{C}$ be the maximal 
    $(2,0)$-compression space in standard coordinates, consisting of all matrices with a zero last column. Then $(\overline{C})^\perp$ is 
    is spanned by $\ve_1 \ve_3^\top$, $\ve_2 \ve_3^\top, \ve_3 \ve_3^\top$. By \Cref{lem:compression spaces std form}, a general 
    $(2,0)$-compression space $C = \{X \in \C^{3 \times 3} \,: \,  BXA \in \overline{C}\}$ for a pair of fixed invertible matrices $A,B$. Therefore, $0 = \langle BXA,\ve_1 \ve_3^\top \rangle = \langle BXA,\ve_2 \ve_3^\top \rangle  = \langle BXA,\ve_3 \ve_3^\top \rangle$ for all $X \in C$. Taking the first dot product:
    \[
    0 = \langle BXA,\ve_1 \ve_3^\top \rangle = \textup{Tr}(A^\top X^\top B^\top \ve_1 \ve_3^\top) = 
    \textup{Tr}(X^\top (B^\top \ve_1) (A \ve_3)^\top) = \langle X,(B^\top \ve_1) (A \ve_3)^\top \rangle. 
    \]
    Setting $A\ve_i = \va_i$ and  $B^\top\ve_j = \vb_j$, we get that $C^{\perp} = \textup{Span}\{\vb_1 \va_3^\top, \vb_2 \va_3^\top, \vb_3 \va_3^\top\}$.  Since $B$ is invertible, $\vb_1, \vb_2, \vb_3$ are independent, and no row of $A$ is all zero. Hence,  $C^\perp$ is spanned by $3$ independent rank one matrices and as a projective space, $\dim(C^\perp) = 2$. 
    Further, any linear combination of these three rank one matrices looks like $ \lambda \vb_1 \va_3^\top + \mu \vb_2 \va_3^\top + \nu \vb_3 \va_3^\top = ( \lambda \vb_1 + \mu \vb_2 + \nu \vb_3) \va_3^\top$ which is a rank one matrix. Therefore, all matrices in $C^\perp$ have rank one. The proof for $(0,2)$-compression spaces is analogous.
        
    \item By the same reasoning as in the previous case, if $C$ is a $(1,1)$-compression space, then 
    $C^\perp = \textup{Span} \{ \vb_1\va_1^\top, $ $\vb_1\va_2^\top, \vb_2 \va_1^\top, \vb_2\va_2^\top \}$. 
    Since $C$ is maximal, $\dim(C^\perp) = 3$ as a projective space, and hence $C^\perp \cong \P^3$. Since the linear combinations 
    $$
    \lambda \gamma  \vb_1 \va_1^\top + \lambda \nu  \vb_1 \va_2^\top + \mu \gamma \vb_2 \va_1^\top + \mu \nu  \vb_2 \va_2^\top = ( \lambda \vb_1 +  \mu \vb_2) (\gamma \va_1 + \nu  \va_2)^\top
    $$ 
    are rank one matrices, there is a 
    $\P^1 \times \P^1$ worth of rank $1$ matrices in $C^\perp$ which forms a subvariety of dimension two. By  \Cref{lem:compression spaces and spans}, 
    there are no more rank one matrices in $C^\perp$.
\end{enumerate}
\end{proof}

From now on we will denote a maximal compression space of rank at most two  
by $C$ and a subspace of it by $C'$. 
If a non-maximal compression space $C'$ appears as an irreducible component of $\RR_2 \cap \mathcal{L}_\mathcal{P}$, 
then $(C')^\perp$ will be spanned by the generators of $C^\perp$ along with some 
{\em data matrices} $\vv_i \vu_i^\perp$ as in \Cref{ex:compression space ex1}.
In particular a basis of $(C')^\perp$ is the union of the basis for $C^\perp$ 
from \Cref{lem:orthogonal complements of compression spaces} and some data matrices.

\begin{lemma}
\label{lem:fit rank one matrices}
Suppose for all distinct indices $i,j,k$,   $\bf{u}_i, \bf{u}_j, \bf{u}_k$ or 
$\bf{v}_i, \bf{v}_j, \bf{v}_k$ are in general position. Let $C'$ be a $(k_1,k_2)$-compression space contained in a maximal $(k_1,k_2)$-compression space $C$. 
\begin{enumerate}
\item If $C'$ is of type $(2,0)$ or $(0,2)$, then $C^\perp$ contains at most one data matrix. 
\item If $C'$ is of type $(1,1)$, then $C^\perp$ contains at most two data matrices. 
\end{enumerate}

\end{lemma}

\begin{proof}
\begin{enumerate}
    \item If $C'$ is a (2,0) compression space, then by \Cref{lem:orthogonal complements of compression spaces}, $C^\perp = \Span \{\bf{b}_1\bf{a}^\top, \bf{b}_2\bf{a}^\top, \bf{b}_3\bf{a}^\top\}$ for some $\bf{a} \neq 0$ and $\bf{b}_1, \bf{b}_2, \bf{b}_3$ which are linearly independent. By \Cref{lem:compression spaces and spans}, $\bf{v}_i\bf{u}_i^\top \in C^\perp$ if and only if $\bf{u}_i \in \Span\{\bf{a}\}$. Since point pairs are distinct and have last coordinate equal to one, at most one $\bf{u}_i \in \Span\{\bf{a}\}$, hence at most one $\bf{v}_i\bf{u}_i^\top \in C^\perp$. The case $(0,2)$ is argued similarly.
    \item If $C'$ is a (1,1) compression space, then by \Cref{lem:orthogonal complements of compression spaces} $C^\perp =\Span\{\vb_1\va_1^\top, \vb_1\va_2^\top, \vb_2 \va_1^\top, \vb_2\va_2^\top\}$ for some linearly independent $\bf{a}_1, \bf{a}_2$ and $\bf{b}_1, \bf{b}_2$.
    By \Cref{lem:compression spaces and spans}, $\bf{v}_i\bf{u}_i^\top \in C^\perp$ if and only if $\bf{u}_i \in \Span\{\bf{a}_1, \bf{a}_2\}$ and $\bf{v}_i \in \Span\{\bf{b}_1, \bf{b}_2\}$. If there are three data matrices in $C^\perp$, say corresponding to indices $i,j,k$, then $\bf{u}_i, \bf{u}_j, \bf{u}_k \in \Span\{\bf{a}_1, \bf{a}_2\}$ and $\bf{v}_i, \bf{v}_j, \bf{v}_k \in \Span\{\bf{b}_1, \bf{b}_2\}$. However this contradicts our general position assumption, so we conclude that $C^\perp$ may contain at most two data matrices. 
     
\end{enumerate}

\end{proof}



We need one final lemma to show that the general position assumption in \Cref{thm:triples general implies irreducible} prevents the existence of 
codimension one compression spaces of rank at most two in $\mathcal{L}_\mathcal{P}$ when $|\mathcal{P}| = 4$. 
The following simple fact about our input point pairs will be useful.

\begin{lemma} \label{lem:simple fact}
If for $i \neq j$, $\vw \vu_i^\top = \vv_j \vu_j^\top$ then $\vw = \vv_j$ and $\vu_i = \vu_j$. Similarly, if for $i \neq j$, $\vv_i \vu_i^\top = \vv_j \vw^\top$ then 
$\vw = \vu_i$ and $\vv_i = \vv_j$.
\end{lemma}
\begin{proof}
Since the $\vu_i$'s have last coordinate 1, if $\vw \vu_i^\top = \vv_j \vu_j^\top$, 
the last columns of the two matrices are $\vw$ and $\vv_j$. Therefore, $\vw = \vv_j$ and hence $\vu_i = \vu_j$. The second statement is proved similarly.
\end{proof}

\begin{lemma}
\label{lem:orthogonal complement overlap}
 Suppose $|\mathcal{P}| = 4$  and 
 for all distinct indices $i,j,k$,   $\bf{u}_i, \bf{u}_j, \bf{u}_k$ or 
 $\bf{v}_i, \bf{v}_j, \bf{v}_k$ are in general position. Let $C'$ be a  $(k_1,k_2)$-compression space of rank at most two and codimension one in $\mathcal{L}_\mathcal{P}$, and let $C$ be a maximal $(k_1,k_2)$-compression space containing $C'$. 
\begin{enumerate}
    \item If $C'$ is of type $(2,0)$ or $(0,2)$, then $C^\perp$ contains at least two data matrices. 
    \item If $C'$ is of type $(1,1)$, then $C^\perp$ contains at least three data matrices.
\end{enumerate}

\end{lemma}

\begin{proof} 
By the general position assumption, at least three of the data matrices are linearly independent, and hence $\dim(\mathcal{L}_\mathcal{P}) = 4 $ or $5$.
\begin{enumerate}
    \item Suppose $C'$ is a maximal $(2,0)$-compression space of rank at most two and codimension one in $\mathcal{L}_\mathcal{P}$. Then by \Cref{lem:orthogonal complements of compression spaces}, 
    $(C')^\perp = \textup{Span}\{ \vb_1 \va^\top, \vb_2 \va^\top, \vb_3 \va^\top \}$ and $\dim(C') = 5$. This implies that $\dim(\mathcal{L}_\mathcal{P}) = 6$ which is a contradiction.

    So suppose $C'$ is a non-maximal $(2,0)$-compression space of rank at most two and codimension one in $\mathcal{L}_\mathcal{P}$, and $C$ is a maximal $(2,0)$-compression space containing $C'$.
    Define 
    $$M_C = \begin{pmatrix} \vb_1 \va^\top \\  \vb_2 \va^\top \\ \vb_3 \va^\top \end{pmatrix}, \,\,\, M_D = \begin{pmatrix} \vv_1 \vu_1^\top \\  \vv_2 \vu_2^\top \\ \vv_3 \vu_3^\top \\ \vv_4 \vu_4^\top \end{pmatrix}, \,\,\, \textup{ and } \,\,\, M = \begin{pmatrix} M_C \\ M_D \end{pmatrix}.$$
     \begin{enumerate}
        \item $\dim(\mathcal{L}_\mathcal{P}) = 5$: Then 
        $\dim(C') = 4$ and $(C')^\perp$ has a basis of cardinality four. 
      We may assume without loss of generality that 
      $(C')^\perp 
        = \textup{Span}\{ \vb_1 \va^\top, \vb_2 \va^\top, \vb_3 \va^\top, \vv_1 \vu_1^\top\}$. 
        By \Cref{lem:compression spaces and spans}, $\vu_1$ is not a multiple of $\va$. 
        Since $\dim(\mathcal{L}_\mathcal{P}) = 5$, $\rank(M_D)=3$, and since $C'$ has codimension one in $\mathcal{L}_\mathcal{P}$, 
        $\rank(M) = 4$ and so $\vv_2 \vu_2^\top, \vv_3 \vu_3^\top, \vv_4 \vu_4^\top \in (C')^\perp.$  
        We need to argue that for $i > 1$, $\vv_i \vu_i^\top \in 
        C^\perp = \textup{Span}\{ \vb_1 \va^\top, \vb_2 \va^\top, \vb_3 \va^\top \}$.
        Suppose for $i > 1$, $\vv_i \vu_i^\top$ equals the linear combination
            \begin{align} \label{eq:lincomb}
            \alpha_1 \vb_1 \va^\top + \alpha_2 \vb_2 \va^\top  + \alpha_3 \vb_3 \va^\top + \beta \vv_1 \vu_1^\top = \begin{pmatrix} \sum \alpha_j \vb_j & \beta \vv_1\end{pmatrix} \begin{pmatrix} \va^\top \\ \vu_1^\top \end{pmatrix}. 
        \end{align}
        If $\beta =0$ then we are done, so suppose $\beta \neq 0$. If all the $\alpha_j$'s are zero, then 
    $\vv_i \vu_i^\top$ is a multiple of $\vv_1 \vu_1^\top$ which is impossible by 
    \Cref{lem:simple fact}. 
    If at least one $\alpha_j \neq 0$, then \eqref{eq:lincomb} has rank two 
    unless $\vv_1 \sim \sum \alpha_j \vb_j$. In this case, \eqref{eq:lincomb} looks like 
    $\vv_1 (\gamma \va^\top + \delta \vu_1^\top)$. 
    Again, by \Cref{lem:simple fact}, 
    this cannot equal $\vv_i \vu_i^\top$ for any $i=2,3,4$. Therefore, it must be 
    that for $i > 1$, $\vv_i \vu_i^\top \in 
        C^\perp$.

    \item $\dim(\mathcal{L}_\mathcal{P}) = 4$: In this case,
    $\dim(C') = 3$, $\rank(M_D) = 4$ and 
        $\rank(M) = 5$. Without loss of generality, 
        $\vv_3 \vu_3^\top, \vv_4 \vu_4^\top \in (C')^\perp 
        = \textup{Span}\{ \vb_1 \va^\top, \vb_2 \va^\top, \vb_3 \va^\top, \vv_1 \vu_1^\top, \vv_2 \vu_2^\top\}$. We need to argue that 
        $\vv_3 \vu_3^\top, \vv_4 \vu_4^\top \in \textup{Span}\{\vb_1 \va^\top, \vb_2 \va^\top, \vb_3 \va^\top \}$. Suppose for $i\in \{3,4\}$, 
        $\vv_i \vu_i^\top$ is a linear combination of the form 
        \begin{align} \label{eq:lincomb2}
            \alpha_1 \vb_1 \va^\top + \alpha_2 \vb_2 \va^\top  + \alpha_3 \vb_3 \va^\top + \beta_1 \vv_1 \vu_1^\top + \beta_2 \vv_2 \vu_2^\top = \begin{pmatrix} 
            \sum \alpha_j \vb_j & \beta_1\vv_1 & \beta_2\vv_2 \end{pmatrix} \begin{pmatrix}
            \va^\top \\ \vu_1^\top \\ \vu_2^\top \end{pmatrix}.
        \end{align}
        If all the $\alpha_i$'s are 0 then $\vv_i  \vu_i^\top$ lies in the span of the other two data matrices which contradicts that $\rank(M_D) = 4$. If all $\beta$'s are $0$ then we are done. So assume that at least one 
        $\alpha_j$ and one $\beta_j$ are non-zero. 
        If only one $\beta_j$ is non-zero, the claim follows from the same argument as in the previous case. So suppose $\beta_1, \beta_2 \neq 0$. By \Cref{lem:compression spaces and spans}, neither $\vu_1$ nor $\vu_2$ are multiples of $\va$ which means that the second matrix in the product has rank at least two. The first matrix also has rank at least two since 
        $\vv_1$ and $\vv_2$ are linearly independent and hence \eqref{eq:lincomb2} has rank at least two and cannot equal $\vv_i \vu_i^\top$ for $i=3,4$.

    \end{enumerate}
     The $(0,2)$ case is argued similarly.
    
  \item Suppose $\RR_2 \cap \mathcal{L}_\mathcal{P}$ has an irreducible component $C'$ which is a $(1,1)$-compression space.
  \begin{enumerate}
  \item $\dim(\mathcal{L}_\mathcal{P}) = 5$: In this case, 
  $\dim(C') = 4$ and $C'=C$ is a maximal compression space. Therefore, 
  $(C')^\perp = C^\perp = 
  \textup{Span}\{\vb_1 \va_1^\top, \vb_1 \va_2^\top, \vb_2 \va_1^\top, \vb_2 \va_2^\top \}$ and $\rank(M) =4$. 
  The matrix $M_C$ now has rows $\vb_1 \va_1^\top, \vb_1 \va_2^\top, \vb_2 \va_1^\top, \vb_2 \va_2^\top $. Since it already has rank four, it must be that all the data matrices are in $C^\perp$. They are in fact in the 
  $\P^1 \times \P^1$ subvariety of $C^\perp$ containing all the rank one matrices.
    
  \item $\dim(\mathcal{L}_\mathcal{P}) = 4$: 
  In this case $\dim(C') = 3$ which means that without loss of generality, $(C')^\perp = \textup{Span}\{ 
  \vb_1 \va_1^\top, \vb_1 \va_2^\top, \vb_2 \va_1^\top, \vb_2 \va_2^\top, \vv_1 \vu_1^\top\}$.
  Also, $\rank(M) =5$ and hence 
  $\vv_2 \vu_2^\top, \vv_3 \vu_3^\top, \vv_4 \vu_4^\top 
  \in (C')^\perp$. We need to argue that 
  these matrices in fact lie in 
  $C^\perp = \textup{Span}\{\vb_1 \va_1^\top, \vb_1 \va_2^\top, \vb_2 \va_1^\top, \vb_2 \va_2^\top\}$. 
  Suppose for $i \in \{2,3,4\}$, $\vv_i \vu_i^\top$ is a linear combination of the form
  \begin{align} \label{eq:lincomb3}
      \alpha_{11}\vb_1 \va_1^\top + \alpha_{12} \vb_1 \va_2^\top 
      + \alpha_{21}\vb_2 \va_1^\top + \alpha_{22}\vb_2 \va_2^\top + \beta \vv_1 \vu_1^\top = \begin{pmatrix} \vb_1 &  \vb_2 &  \vv_1 \end{pmatrix} 
      \begin{pmatrix} \alpha_{11} \va_1^\top + \alpha_{12} \va_2^\top \\ 
      \alpha_{21} \va_1^\top + \alpha_{22} \va_2^\top \\ \beta \vu_1^\top \end{pmatrix}
  \end{align}
  with $\beta \neq 0$. As before, some $\alpha_{ij}$ must also be non-zero. 
  By \Cref{lem:compression spaces and spans}, $\vv_1 \not \in \textup{Span}\{\vb_1, \vb_2\}$ or $\vu_1 \not \in \textup{Span}\{ \va_1, \va_2\}$. Suppose $\vv_1 \not \in \textup{Span}\{\vb_1, \vb_2\}$. Then the first matrix in \eqref{eq:lincomb3} 
  has rank two and hence \eqref{eq:lincomb3} has rank two unless 
  all rows of the second matrix are multiples of $\vu_1$. Then as in the previous cases, \eqref{eq:lincomb3} cannot equal 
  $\vv_i \vu_i^\top$ for any $i=2,3,4$. If $\vu_1 \not \in \textup{Span}\{ \va_1, \va_2\}$ then the second matrix 
  in \eqref{eq:lincomb3} has rank two and hence the linear combination has rank two unless all columns of the first matrix are multiples of $\vv_1$. Again by \Cref{lem:simple fact}, 
  \eqref{eq:lincomb3} cannot be $\vv_i \vu_i^\top$ for any $i=2,3,4$.
  \end{enumerate}

\end{enumerate}

\end{proof}



\begin{proof}[Proof of \Cref{thm:triples general implies irreducible}]
If $\RR_2 \cap \mathcal{L}_\mathcal{P}$ is reducible, then by \Cref{thm:codim 1 compression} there is a codimension one compression space of rank at most two in $L_\mathcal{P}$. However, if for all distinct indices $i,j,k$,   $\bf{u}_i, \bf{u}_j, \bf{u}_k$ or $\bf{v}_i, \bf{v}_j, \bf{v}_k$ are in general position, 
then \Cref{lem:fit rank one matrices} and \Cref{lem:orthogonal complement overlap} 
together prohibit the existence of a codimension one compression space of rank at most two in $\mathcal{L}_\mathcal{P}$, proving the theorem.
\end{proof}

\subsection{Chiral Reconstructions for Four Point Pairs}
\label{subsec:chiral reconstructions k=4}

We now have the tools to determine when a set of four distinct point pairs has a chiral reconstruction. Up to permuting the two images, every set of four point pairs belong to one of the combinatorial types shown in
\Cref{fig:combtypes}.

\begin{figure}[ht] 
\begin{subfigure}{.25\textwidth}
  \centering
  \includegraphics[width=\linewidth]{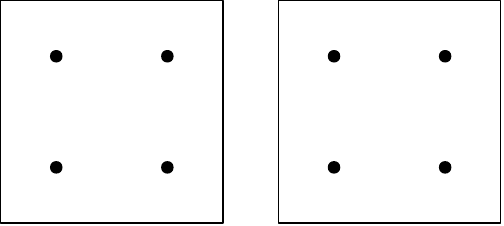} 
  \caption{Both $\mathcal{U}$ and $\mathcal{V}$ are in general position.}
  \label{fig:gengen}
\end{subfigure}
\hspace{0.05\textwidth}
\begin{subfigure}{.25\textwidth}
  \centering
  \includegraphics[width=\linewidth]{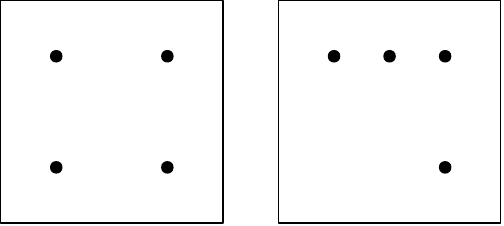}
\caption{$\mathcal{U}$ general position, three $\bf{v}_i$ collinear.}
\label{fig:gen3col}
\end{subfigure}
\hspace{0.05\textwidth}
\begin{subfigure}{.25\textwidth}
  \centering
  \includegraphics[width=\linewidth]{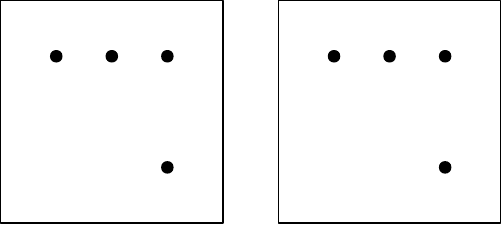}
\caption{Three $\bf{u}_i$ collinear, three $\bf{v}_i$ collinear.}
\label{fig:3col3col}
\end{subfigure}
\hspace{0.05\textwidth}
\begin{subfigure}{.25\textwidth}
  \centering
  \includegraphics[width=\linewidth]{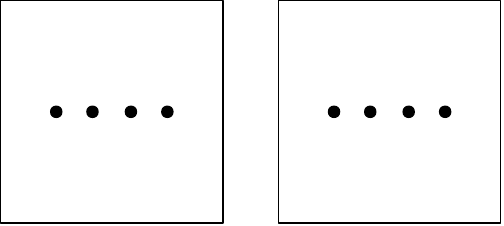}
 \caption{Four $\bf{u}_i$ collinear, four $\bf{v}_i$ collinear.}
 \label{fig:4col4col}
\end{subfigure}
\hspace{0.05\textwidth}
\begin{subfigure}{.25\textwidth}
  \centering
  \includegraphics[width=\linewidth]{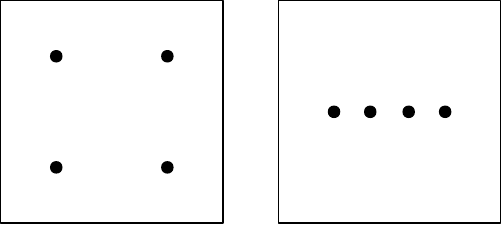}
 \caption{$\mathcal{U}$ general position, four $\bf{v}_i$ collinear.}
 \label{fig:gen4col}
\end{subfigure}
\hspace{0.05\textwidth}
\begin{subfigure}{.25\textwidth}
  \centering
  \includegraphics[width=\linewidth]{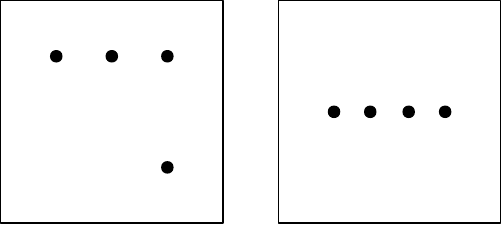}
 \caption{Three $\bf{u}_i$ collinear, four $\bf{v}_i$ collinear.}
 \label{fig:3col4col}
\end{subfigure}
\caption{\label{fig:combtypes}}
\end{figure}

We can further group these types by the \emph{rank} of the points in $\mathcal{U} 
= \{\vu_1, \ldots, \vu_4\}$ and $\mathcal{V} =\{\vv_1, \ldots, \vv_4\}$.

\begin{definition}
We say that $\rank(\mathcal{U})= r$ if and only if the $3 \times 4$ matrix $U$ with columns $\bf{u}_1, \dots, \bf{u}_4$ has rank $r$. Equivalently, $\rank(\mathcal{V})= r$ if and only if the $3 \times 4$ matrix $V$ with columns $\bf{v}_1, \dots, \bf{v}_4$ has rank $r$.
\end{definition}

Rank captures the geometry of $\mathcal{U}$ and $\mathcal{V}$. Indeed, if $rank(\mathcal{U}) = 1$, then all $\bf{u}_i$ are coincident as points in $\P^2$. If $rank(\mathcal{U}) = 2$, then all $\bf{u}_i$ are collinear in $\P^2$, and if $rank(\mathcal{U}) = 3$, then some 
three $\bf{u}_i$ are non-collinear in $\P^2$. The assumption that points in $\mathcal{U}$ and 
$\mathcal{V}$ are distinct implies that $\rank(\mathcal{U}), \rank(\mathcal{V}) \ge 2$. The spirit of \Cref{thm:4 chiral exists}, the main result of this section,
is that 
when $\mathcal{U}$ and $\mathcal{V}$ have similar geometry, then 
$\mathcal{P}$ has a chiral reconstruction. In particular, a chiral reconstruction exists for the combinatorial types in \Cref{fig:gengen}, \Cref{fig:gen3col}, and  \Cref{fig:3col3col} where $\rank(\mathcal{U}) = \rank(\mathcal{V}) = 3$ and for the type in \Cref{fig:4col4col} where $\rank(\mathcal{U}) = \rank(\mathcal{V}) = 2$. We will present examples of configurations of the types in \Cref{fig:gen4col} and \Cref{fig:3col4col} for which a chiral reconstruction can fail to exist.

\begin{theorem}
\label{thm:4 chiral exists}
If $|\mathcal{P}|=4$ and $\rank(\mathcal{U}) = \rank(\mathcal{V})$, then $\mathcal{P}$ has a chiral reconstruction. 
\end{theorem}


\begin{proof}
We break the proof into three parts:

\begin{enumerate}
    \item Suppose $\rank(\mathcal{U}) = \rank(\mathcal{V}) = 2$. Then the points in both $\mathcal{U}$ and $\mathcal{V}$ are collinear as in \Cref{fig:4col4col}, even as points in $\R^3$. Assume without loss of generality that the 
    $\vv_i$ points appear in the order $\bf{v}_1, \bf{v}_2, \bf{v}_3, \bf{v}_4$ along the affine line $L$ they span in $\R^3$.
    The ordering of $\bf{v}_i$ induces an ordering of the $\bf{u}_i$.
    Let $l, r \in \{1,2,3,4\}$ be such that for all $i \in \{1,2,3,4\}$, $\bf{u}_i \in \cone(\bf{u}_l, \bf{u}_r)$, and  define $G \in GL_3$ such that $G\bf{u}_l = \bf{v}_1$ and $G\bf{u}_r = \bf{v}_4$. This forces $\bf{v}_i, G\bf{u}_i \in \cone(\bf{v}_1 , \bf{v}_4)$ for all $i$. For a $\bf{t} \in L \setminus \cone(\bf{v}_1 , \bf{v}_4)$,  $(\bf{t} \times \bf{v}_i)^\top(\bf{t} \times G\bf{u}_i) > 0$ by \Cref{lem:t cross quad product}. Then, \Cref{lem:gigj is quad product} implies that $g_i([\bf{t}]_\times G) > 0$ for all $i$, proving the result.   

\item Suppose $\rank(\mathcal{U}) = \rank(\mathcal{V}) = 3$ and 
for all $i,j,k$,   $\bf{u}_i, \bf{u}_j , \bf{u}_k$ or $\bf{v}_i, \bf{v}_j, \bf{v}_k$ are in general position.
Then by \Cref{thm:triples general implies irreducible}, $\RR_2 \cap L_\mathcal{P}$ is irreducible, and hence by \Cref{thm:walk away from walls}, it suffices to show that there is a smooth rank two matrix $X \in \mathcal{R}_2 \cap \mathcal{L}_\mathcal{P}$ for which $\mathcal{P}_{I(X)}$ has a chiral reconstruction.

We first claim that there is an ordering of point pairs such that $\bf{u}_1, \bf{u}_2, \bf{u}_3$ are in general position and $\bf{v}_4$ does not lie on any of the lines $L^{12}$, $L^{13}$, and $L^{23}$ where $L^{ij}$ is the line spanned by 
$\vv_i$ and $\vv_j$. Indeed, since $\rank(\mathcal{V}) = 3$, there is an ordering of the point pairs so that $\bf{v}_4$ does not lie on any of the lines $L^{12}$, $L^{13}$, or $L^{23}$.
If $\bf{u}_1, \bf{u}_2, \bf{u}_3$ are in general position for this ordering, we are done. Otherwise, $\bf{u}_1, \bf{u}_2, \bf{u}_3$ collinear and so  
$\bf{v}_1, \bf{v}_2, \bf{v}_3$ must be in general position by assumption, and hence $\bf{v}_1, \bf{v}_2, \bf{v}_3,\bf{v}_4$ are in general position. Also, since  $\rank(\mathcal{U})  = 3$, $\bf{u}_1, \bf{u}_2, \bf{u}_4$ must be non-collinear. Since $\bf{v}_3$ cannot be on the lines $L^{12}, L^{14}, L^{24}$, swapping point pairs with indices 3 and 4 gives the desired result.

Reorder point pairs so that $\bf{u}_1, \bf{u}_2, \bf{u}_3$ are in general position and $\bf{v}_4$ does not lie on any of $L^{12}$, $L^{13}$, and $L^{23}$.
Then choosing $\bf{e}_2 = \bf{v}_4$ forces
$\det[ \bf{v}_1 \, \bf{v}_2 \, \bf{e}_2], \,\,
\det[ \bf{v}_1 \, \bf{v}_3 \, \bf{e}_2], \,\,
\det[ \bf{v}_2 \, \bf{v}_3 \, \bf{e}_2]$
to all be non-zero. Removing $\bf{u}_4$ from $\mathcal{U}$ leaves three points in general position, so we may pick $\bf{e}_1$ by \Cref{lem:match chirotope} so that 
$\det[ \bf{u}_i \, \bf{u}_j \, \bf{e}_1]\det[ \bf{v}_i \, \bf{v}_j \, \bf{e}_2] >0$ 
for all $1 \leq i < j \leq 3$. Since the positivity of these determinant products is an open condition, there is a open polyhedral cone
from which such an $\bf{e}_1$ can be chosen. Consider the system 
\[
X \bf{e}_1 = 0, \,\, \bf{v}_4^\top X = 0, \,\,\bf{v}_i^\top X \bf{u}_i = 0, \,\,i=1,2,3,4
\]
which consists of at most eight linearly independent equations. Therefore, this system has a solution $X \in \P^8$. From \Cref{lem:smooth points}, we know the tangent space to $\mathcal{R}_2$ at $X$ has normal $\bf{v}_4 \bf{e}_1^\top$ and $X$ is a smooth point of the epipolar variety if $\bf{v}_4 \bf{e}_1^\top$ does not lie in $\Span\{\bf{v}_i\bf{u}_i^\top, \,\,i=1,2,3,4\} $. 
If $\bf{v}_4 \bf{e}_1^\top \in \Span\{\bf{v}_i\bf{u}_i^\top, \,\,i=1,2,3,4\}$, there is a minimal subset of the five matrices, including $\bf{v}_4 \bf{e}_1^\top$ that are 
dependent. Placing the vectorizations of these matrices in the rows of a matrix, we get that all its maximal minors are zero. Each of these minors, set to zero, is a linear equation in $\bf{e}_1$ and together they cut out a subspace in $\P^2$. Since we were able to choose $\bf{e}_1$ from an open polyhedral cone, there is a choice of $\bf{e}_1$ that avoids this subspace and the rank one variety. Hence, $X$ can be 
chosen (by choosing $\bf{e}_1$ appropriately) to be a smooth rank two fundamental matrix and we are done by 
\Cref{thm:walk away from walls}.

\item Suppose $\rank(\mathcal{U}) = \rank(\mathcal{V})=3$ and there 
exist $i,j,k$ for which both $\bf{u}_i, \bf{u}_j , \bf{u}_k$ and $\bf{v}_i, \bf{v}_j, \bf{v}_k$ are collinear. 
Without loss of generality we may assume that 
$\{i,j,k\}=\{1,2,3\}$, 
and that $\bf{v}_1, \bf{v}_2, \bf{v}_3$ appear in this order on the affine line $L$ they span in $\R^3$.
Let $l, r \in \{1,2,3\}$ be such that $\vu_1,\vu_2,\vu_3 \in \cone(\bf{u}_l, \bf{u}_r)$, and define $G$ by $G(\bf{u}_l) = \bf{v}_1$, $G(\bf{u}_r) = \bf{v}_3$, and $G(\bf{u}_4) = \bf{v}_4$. 
Then for a $\bf{t} \in L \setminus \cone(\bf{v}_1, \bf{v}_3)$, set 
$X = [\bf{t}]_\times G$. Since, $\bf{v}_i, \bf{t}, $ and $G\bf{u}_i$ are collinear for each $i$, $X$ satisfies all the epipolar equations $\bf{v}_i^\top X \bf{u}_i = 0$. Also, since  $\bf{v}_i, G\bf{u}_i \in \cone(\bf{v}_1 , \bf{v}_3)$ for $i = 1,2,3$, $g_1(X), g_2(X), g_3(X) > 0$ by \Cref{lem:gigj is quad product} and \Cref{lem:t cross quad product}. Finally, 
$
g_4(X) = (\bf{t} \times \bf{v}_4)^\top (\bf{t} \times G \bf{u}_4) = (\bf{t} \times \bf{v}_4)^\top (\bf{t} \times \bf{v}_4) > 0, 
$
and so $X$ strictly satisfies all chiral epipolar inequalities and  $\mathcal{P}$ has a chiral reconstruction by \Cref{thm:chiral F existence gigj}. 

\end{enumerate}
\end{proof}

We conclude by showing that for four point pairs, \Cref{thm:4 chiral exists} is best possible in the following sense. 

\begin{theorem}
\label{thm:4 no chiral}
For the two combinatorial types where
$\rank(\mathcal{U}) \neq \rank(\mathcal{V})$, there are instances of $\mathcal{P}$ without a chiral reconstruction.
\end{theorem}

We give examples to prove \Cref{thm:4 no chiral}. 
Recall that the epipolar line homography of \Cref{thm:epipolar homography} cannot send coincident lines to distinct lines or vice versa. Therefore if $X$ is a $\mathcal{P}$-regular fundamental matrix with right and left kernels generated by $\bf{e}_1$ and $\bf{e}_2$, respectively, then $\bf{e}_1,\bf{u}_i,\bf{u}_j$ are collinear 
 if and only if $\bf{e}_2,\bf{v}_i,\bf{v}_j$ are collinear.

\begin{example}
\label{ex:3 in line 4 in line}
Consider the arrangement in \Cref{fig:3line4line} where $\rank(\mathcal{U}) = 3 \neq \rank(\mathcal{V}) = 2$.

\begin{figure}[ht] 
\begin{subfigure}{.2\textwidth}
  \centering
  \includegraphics[width=\linewidth]{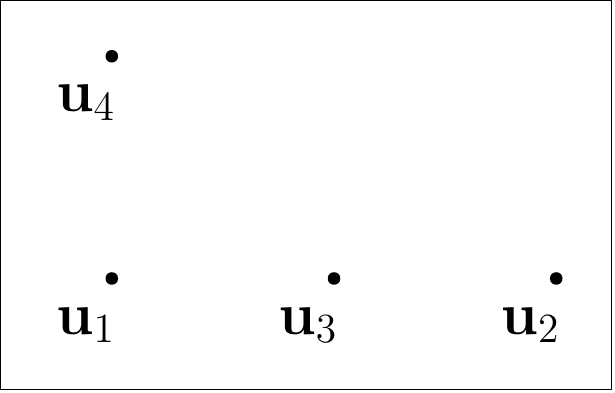} 
\end{subfigure}
\hspace{0.02\textwidth}
\begin{subfigure}{.45\textwidth}
\[ U = \begin{pmatrix}1&3&2&1\\
                    0&0&0&1\\
                    1&1&1&1\\
      \end{pmatrix} \quad V = \begin{pmatrix}1&2&3&4\\
                    0&0&0&0\\
                    1&1&1&1\\
      \end{pmatrix}
      \]
      \end{subfigure}
\hspace{0.02\textwidth}
\begin{subfigure}{.28\textwidth}
  \centering
  \includegraphics[width=\linewidth]{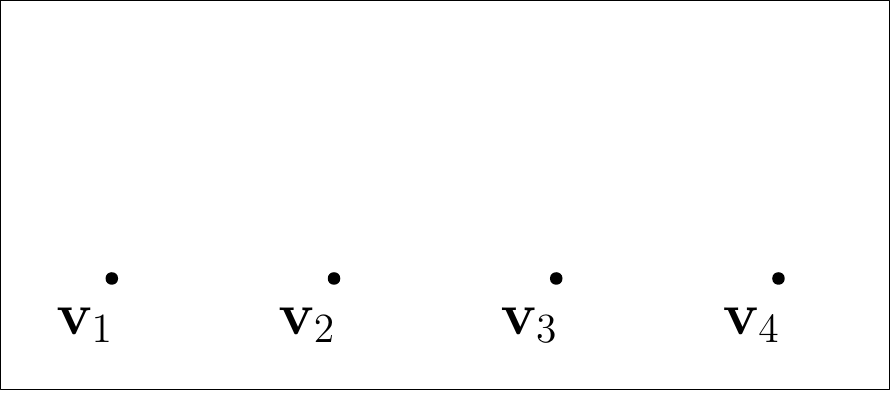} 
\end{subfigure}
\caption{\label{fig:3line4line}}
\end{figure}

Let $L$ be the line in $\P^2$ spanned by $(0,0,1)^\top$ and $(1,0,1)^\top$. Suppose there is a 
fundamental matrix $X$ for $\mathcal{P}$ with $\bf{e}_2^\top X = 0 = X \bf{e}_1$, and suppose $\bf{e}_2 = (e_{21}, e_{22}, e_{23})^\top$ and $\bf{e}_1 = (e_{11}, e_{12}, e_{13})^\top$. 
Then $\bf{e}_2 \not \in L$ (equivalently, $e_{22} \neq 0$) because then the set of lines $\bf{e}_2(\bf{v}_1, \bf{v}_2, \bf{v}_3, \bf{v}_4)$ consists of a repeated line, while the set of lines $\bf{e}_1(\bf{u}_1, \bf{u}_2, \bf{u}_3, \bf{u}_4)$ contains at least two distinct lines for any choice of $\bf{e}_1$. If $\bf{e}_2 \in \P^2 \setminus L$, then 
$\bf{e}_2(\bf{v}_1, \bf{v}_2, \bf{v}_3, \bf{v}_4)$ consists 
of four distinct lines, and hence $\bf{e}_1(\bf{u}_1, \bf{u}_2, \bf{u}_3, \bf{u}_4)$ must also have four distinct lines. 
In particular, $\bf{e}_1 \notin L$, or equivalently, $e_{12} \neq 0$. 
These restrictions imply that 
$
D_{12}(\bf{e}_1, \bf{e}_2) = 2 e_{12} e_{22}, \quad  D_{13}(\bf{e}_1, \bf{e}_2) = 2 e_{12} e_{22}, \quad D_{23}(\bf{e}_1, \bf{e}_2) = - e_{12} e_{22},
$
are all non-zero. It is then clear that they cannot 
all have the same sign. Therefore, 
by \Cref{thm:chiral inactive indices}, there is no chiral reconstruction of 
$\mathcal{P}$.

\end{example}

The epipolar variety in \Cref{ex:3 in line 4 in line} is reducible; $W^{\bf{v}_4}$ is a linear component containing 
the wall $W_{\bf{u}_2}$.
Let $\bf{t} = \bf{v}_4$ and define $G \in \GL_3$ such that $G\bf{u}_1 = \bf{v}_1$ and $G\bf{u}_2 = \bf{v}_4$. Then 
$X = [\bf{t}]_\times G$ is a fundamental matrix of $\mathcal{P}$. Indeed, $\bf{v}_4^\top X = 0$ and $X\bf{u}_2 = [\bf{v}_4]_\times G\bf{u}_2 = [\bf{v}_4]_\times \bf{v}_4 = 0$  and therefore $X$ satisfies the second and fourth 
epipolar equations. It is straightforward to check that $\bf{v}_1, \bf{t}, G\bf{u}_1$ are collinear and $\bf{v}_3, \bf{t}, G\bf{u}_3$ are collinear, so $X$ also satisfies the first and third epipolar equations. By construction, $g_2(X) = g_4(X) = 0$, and by \Cref{lem:gigj is quad product} and \Cref{lem:t cross quad product}, $g_1g_3(X) > 0$. Thus, 
$X$ is a fundamental matrix of $\mathcal{P}$ that satisfies all the chiral epipolar inequalities $g_ig_j(X) \ge 0$ for all $1 \leq i < j \leq 4$. However, $X$ lies in the corner $W_{\bf{u}_2}\cap W^{\bf{v}_4}$ and is not $\mathcal{P}$-regular and all neighboring matrices to $X$ on $W^{\bf{v}_4}$ are also $\mathcal{P}$-irregular. Therefore, we cannot use \Cref{thm:walk away from walls} to perturb $X$ to a regular fundamental matrix. 

\Cref{ex:3 in line 4 in line} illustrates why the irreducibility of the epipolar variety 
is needed in \Cref{thm:walk away from walls}. 
The next example shows that a chiral reconstruction may not exist even when the epipolar variety is irreducible.

\begin{example} \label{ex:square4line}
Consider the arrangement in \Cref{fig:square4line} where $\rank(\mathcal{U}) = 3 \neq \rank(\mathcal{V}) = 2$.
\begin{figure}[ht] 
\begin{subfigure}{.15\textwidth}
  \centering
  \includegraphics[width=\linewidth]{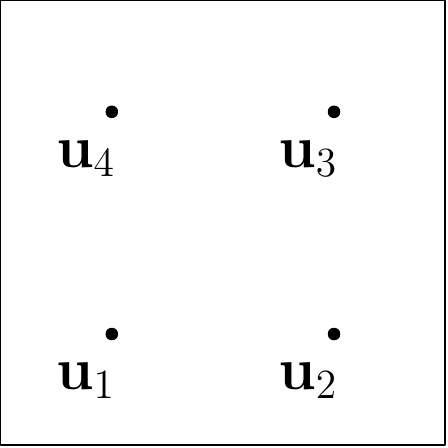}
\end{subfigure}
\hspace{0.03\textwidth}
\begin{subfigure}{.45\textwidth}
\[U = \begin{pmatrix}0&1&1&0\\
                     0&0&1&1\\
                     1&1&1&1\\
      \end{pmatrix} \quad V =  \begin{pmatrix}1&3&2&4\\
                    0&0&0&0\\
                    1&1&1&1\\
      \end{pmatrix}\]
\end{subfigure}
\hspace{0.03\textwidth}
\begin{subfigure}{.28\textwidth}
  \centering
  \includegraphics[width=\linewidth]{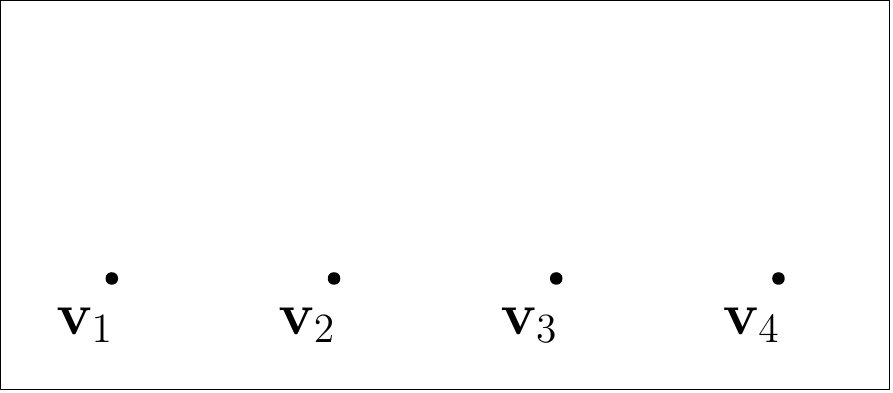} 
\end{subfigure}
\caption{\label{fig:square4line}}
\end{figure}

By \Cref{thm:triples general implies irreducible}, this epipolar variety is irreducible. Suppose $\mathcal{P}$ has a fundamental matrix $X$ with left epipole  $\bf{e}_2=(e_{21},e_{22},e_{23})^\top$ and right epipole $\bf{e}_1=(e_{11},e_{12},e_{13})^\top$. By \Cref{thm:epipolar homography}, $\bf{e}_2$ cannot be on the line spanned by the $\bf{v}$ points since there is no $\bf{e}_1$ for which the set of lines $\bf{e}_1(\bf{u}_1, \ldots, \bf{u}_4)$ consists of coincident lines. For any $\bf{e}_2 = (e_{21},e_{22},e_{23})^\top$ not on the line spanned by the $\bf{v}$ points, the vector 
\[
(\det[\bf{v}_i \, \bf{v}_j \, \bf{e}_2] \,:\, 1 \leq i < j \leq 4) = (2e_{22}, e_{22}, 3e_{22}, -e_{22}, e_{22}, 2e_{22} ) 
\]
which has sign pattern $+++-++$ (up to sign) since $e_{22} \neq 0$, and for any $\bf{e}_1=(e_{11},e_{12},e_{13})^\top$,
\[
(\det[\bf{u}_i \, \bf{u}_j \, \bf{e}_1] \,:\, 1 \leq i < j \leq 4) = (e_{12}, -e_{11}+e_{12}, -e_{11}, -e_{11}+e_{13}, -e_{11}-e_{12}+e_{13}, -e_{12}+e_{13} ).
\]
We can guarantee a chiral reconstruction if 
\[
e_{12} > 0, -e_{11}+e_{12} >  0, -e_{11} >  0, -e_{11}+e_{13} < 0, -e_{11}-e_{12}+e_{13} > 0, -e_{12}+e_{13} > 0.
\]
However, this implies that $e_{12} > e_{11} > e_{13} > e_{12}$ which is a contradiction. Similarly, we cannot achieve the sign pattern $---+--$ either.  What might still be possible is to choose $\bf{e}_1$ so that some of the determinants $\det[\bf{u}_i \, \bf{u}_j \, \bf{e}_1]$ are zero, or equivalently, $\bf{e}_1$ lies on the line joining two of the $\bf{u}$ points. This will not work since then two coincident rays will have to map to two non-coincident rays under the homography in \Cref{thm:epipolar homography}. 
Therefore, $\mathcal{P}$ has no chiral reconstruction.

\end{example}

Unlike in \Cref{ex:3 in line 4 in line}, there is no fundamental matrix for \Cref{ex:square4line} that satisfies all chiral epipolar inequalities. Indeed, \Cref{thm:walk away from walls} implies that there cannot be a smooth $X$ such that $\mathcal{P}_{I(X)}$ has a chiral reconstruction associated to $X$. This rules out the possibility that there is an irregular $X$ with respect to some index that satisfies the chiral epipolar inequalities. 


\section{Five Point Pairs}
\label{sec:fivepoints}

We have seen so far that three point pairs unconditionally have a chiral reconstruction as do four point pairs unless their geometry is special and dissimilar. In this section we will see that five point pairs can fail to have a chiral reconstruction even if both sets of points are in general position. Specific examples of such point pairs were known to Werner \cite{werner2003constraint}. We will show that, in fact, five point pairs do not have a chiral reconstruction with positive probability. We will use \Cref{thm:walk away from walls} to devise a simple test for chirality in this section, and connect to the classical theory of cubic surfaces in the next section. 

Throughout this section, let $\mathcal{P} = \{(\bf{u}_i, \bf{v}_i) : i = 1, \dots, 5\}$ be a set of five generic point pairs in the sense that 
$\mathcal{L}_\mathcal{P}$ is a generic subspace in $\P^8$ of codimension five. In this case, $\mathcal{L}_\mathcal{P}$ misses the four-dimensional variety $\mathcal{R}_1$, and by Bertini's Theorem, the epipolar variety $\mathcal{R}_2 \cap \mathcal{L}_\mathcal{P}$ is a smooth, irreducible cubic surface. Each wall $W_{\bf{u}_i}$ or $W^{\bf{v}_j}$ is a line on $\mathcal{R}_2 \cap \mathcal{L}_\mathcal{P}$. We first see how these lines intersect.

\begin{lemma}
\label{lem:k5corners}
Let $\mathcal{P}$ be a set of five generic point pairs. Then 
\begin{enumerate}
 \item $W_{\bf{u}_i}$ and $W_{\bf{u}_j}$, and similarly 
  $W^{\bf{v}_i}$ and $W^{\bf{v}_j}$, do not intersect for all $i \neq j$.
\item  $W_{\bf{u}_i}$ and $W^{\bf{v}_i}$ do not intersect for all $i$. 
\item The corner $W_{\bf{u}_i} \cap W^{\bf{v}_j}$ consists of a unique real rank two  matrix for $i\neq j$.
\end{enumerate}
\end{lemma}

\begin{proof}
\begin{enumerate}
    \item Recall that for $i \neq j$, $\bf{u}_i \not \sim \bf{u}_j$, and so 
there is no rank two $X$ such that $X\bf{u}_i = 0 = X\bf{u}_j$. 
\item Since $\widetilde{W}_{\bf{u}_i}, \widetilde{W}^{\bf{v}_i} \subset L_{(\bf{u}_i, \bf{v}_i)}$, ${W}_{\bf{u}_i} \cap {W}^{\bf{v}_i}$ is the intersection of the three-dimensional linear space $\widetilde{W}_{\bf{u}_i} \cap \widetilde{W}^{\bf{v}_i}$ with the generic codimension four linear space $\bigcap_{l \neq i} L_{(\bf{u}_l, \bf{v}_l)}$ in $\P^8$, and hence empty. 

\item When $i\neq j$, $W_{\bf{u}_i} \cap W^{\bf{v}_j}$ is the intersection of the three-dimensional linear space $\widetilde{W}_{\bf{u}_i} \cap \widetilde{W}^{\bf{v}_j}$ with the codimension three linear space $\bigcap_{l \neq i,j} L_{(\bf{u}_l, \bf{v}_l)}$ in $\P^8$. This intersection is zero-dimensional and since the defining equations are linear with real coefficients, it consists of a single real matrix $X$. Generically, the intersection will miss $\mathcal{R}_1$, so we conclude that $X$ is a real rank two matrix.
\end{enumerate}
\end{proof}

\begin{remark}
Concretely, when $i \neq j$, the rank two matrix $W_{\bf{u}_i} \cap W^{\bf{v}_j}$ is $X = [\bf{v}_j]_\times H$ where $H$ is the unique homography sending $\bf{u}_i$ to $\bf{v}_j$ and  $\bf{u}_l$ to $\bf{v}_l$ for all $l \neq i,j$. It is easy to verify that this $X$ satisfies the epipolar equations and is the unique point in $W_{\bf{u}_i} \cap W^{\bf{v}_j}$.
\end{remark}

\begin{definition}
Let $C$ be the set of all wall intersections on the epipolar variety:
\[
C = \bigcup_{1\le i , j \le k} W_{\bf{u}_i} \cap W^{\bf{v}_j}.
\]
We call the points in $C$ the {\em corners} of the epipolar variety.
\end{definition}

\Cref{lem:k5corners} shows that $C$ consists of 20 distinct fundamental matrices of $\mathcal{P}$. However, they do not correspond to full projective reconstructions of $\mathcal{P}$ because necessarily the $(\bf{u}_i, \bf{v}_j)$ corner is neither $(\bf{u}_i, \bf{v}_i)$ regular nor $(\bf{u}_j, \bf{v}_j)$ regular. Despite this fact, we argue that checking the signs of the chiral epipolar inequalities at these corners suffices to determine whether $\mathcal{P}$ has a chiral reconstruction.

\begin{theorem}
\label{thm:exist chiral at corner}
Let $\cP$ be a generic set of five point pairs. Then $\mathcal{P}$ has a 
chiral reconstruction if and only if $\mathcal{P}_{I(X)}$ has a chiral reconstruction for one of the 20 corners $X \in C$. 
\end{theorem}

\begin{proof}
Suppose $\mathcal{P}$ has a chiral reconstruction. Then \Cref{thm:chiral F existence gigj} implies that there is a $\mathcal{P}$-regular fundamental matrix $X$ such that $g_ig_j(X) \ge 0$ for all $i,j$. By \Cref{lem:k5corners}, the epipolar variety does not contain the $(\bf{u}_i, \bf{v}_i)$ corner for any $i$. Therefore regularity with respect to all image pairs implies $g_i(X) \neq 0$ for all $i$, meaning $g_ig_j(X) > 0$ for all $i,j$. Let $U$ be the connected component of the semialgebraic subset of the epipolar variety described by $g_ig_j \ge 0$ that contains $X$. The chiral epipolar inequalities only change sign along walls. Either $U$ is the entire epipolar variety or its boundary is contained in the walls. Every wall contains four corners and every interval on a wall that is contained in $U$ is bounded by corners (or unbounded, in which case it contains all four corners on the wall). So $U$ contains corners and every corner $X\in U$ lies in $C$ and corresponds to a chiral reconstruction of the subset $\PP_{I(X)}$ of point pairs.

 Conversely, suppose $\mathcal{P}_{I(X)}$ has a chiral reconstruction associated to some $X \in C$. The epipolar variety $\RR_2\cap L_\mathcal{P}$ is smooth, irreducible, and consists only of rank two matrices. Hence the points in $C$ are smooth fundamental matrices and \Cref{thm:walk away from walls} implies there is a chiral reconstruction of $\mathcal{P}$.
\end{proof}

\Cref{thm:chiral inactive indices} allows us to infer the sign of all non-zero chiral epipolar inequalities at a corner $X \in C$  directly using the input point pairs. The next corollary makes this precise. 

\begin{corollary}
\label{cor:5 point det at corner}
Let $\mathcal{P}$ be a set of five generic point pairs. Then 
$\mathcal{P}$ has a chiral reconstruction if and only if there is some $(\bf{u}_i, \bf{v}_j)$ corner such that $D_{lm}(\bf{u}_i, \bf{v}_j), D_{ln}(\bf{u}_i, \bf{v}_j), $ and $D_{mn}(\bf{u}_i, \bf{v}_j)$ have the same sign for $l,m,n \in [5] \setminus \{i,j\}$. 
\end{corollary}

\begin{proof}
Since $\mathcal{P}$ is generic, no three $\bf{u}_i$ are collinear and no three $\bf{v}_i$ are collinear. Hence, the expressions $D_{lm}(\bf{u}_i, \bf{v}_j),
$ $D_{ln}(\bf{u}_i, \bf{v}_j), $ and $D_{mn}(\bf{u}_i, \bf{v}_j)$ are all non-zero for pairwise distinct $l,m,n \in [k] \setminus \{i,j\}$. 
There is some $(\bf{u}_i, \bf{v}_j)$ corner $X\in C$ such that $D_{lm}(\bf{u}_i, \bf{v}_j), D_{ln}(\bf{u}_i, \bf{v}_j), $ and $D_{mn}(\bf{u}_i, \bf{v}_j)$ have the same sign if and only if $\mathcal{P}_{I(X)} = \mathcal{P}_{\{l,m,n\}}$ has a chiral reconstruction associated to $X$ (\Cref{thm:chiral inactive indices}) if and only if $\mathcal{P}$ has a 
chiral reconstruction (\Cref{thm:exist chiral at corner}).
\end{proof}

\Cref{cor:5 point det at corner} gives an algorithm to check whether five generic point pairs have a chiral reconstruction. We evaluate the sign of three polynomials on the input point pairs at each of the 20 corners. The signs are the same at a corner if and only if it lies in the boundary of the chiral epipolar region of $\mathcal{P}$. We illustrate the procedure on an example.

\begin{example}
\label{ex:Werner}
Let 
\begin{align} \label{eq:werner no chiral}
U =             \left(\begin{array}{ccccc} 0 & 0 & 4 & 2 & 2 \\ 0 & 4 & 0 & 1 & 3 \\ 1 & 1 & 1 & 1 & 1 \end{array}\right)
 \,\,\,\textup{ and } \,\,\,
      V =         \left(\begin{array}{ccccc} 2 & 2 & 4 & 0 & 1 \\ 1 & 3 & 0 & 4 & 1 \\ 1 & 1 & 1 & 1 & 1 \end{array}\right), 
\end{align}
and $\mathcal{P} = \{(\bf{u}_i,\bf{v}_i)\}$ where $\bf{u}_i$ is the $i$th column of $U$ and $\bf{v}_i$ is the $i$th column of V. By \Cref{cor:5 point det at corner} 
and the following table, $\mathcal{P}$ has no chiral reconstruction.

\[\small \left(\begin{array}{c|c|c|c|c}i&j& D_{lm}(\bf{u}_i, \bf{v}_j) & D_{ln}(\bf{u}_i,\bf{v}_j) & D_{mn}(\bf{u}_i, \bf{v}_j) \\ \hline 
    1 & 2 & -16 & -84 & 20\\ 1 & 3 & -32 & -56 & 32\\ 1 & 4 & 64 & 40 & -96\\ 1 & 5 & 112 & -40 & 32\\ 2 & 1 & -16 & -4 & 12\\ 2 & 3 & -32 & 8 & 32\\ 2 & 4 & 64 & -24 & -32\\ 2 & 5 & -16 & 24 & -32\\ 3 & 1 & 16 & -8 & -12\\ 3 & 2 & 16 & 24 & -20\\
    \end{array}\right)
    \left(\begin{array}{c|c|c|c|c}i&j& D_{lm}(\bf{u}_i, \bf{v}_j) & D_{ln}(\bf{u}_i,\bf{v}_j) & D_{mn}(\bf{u}_i, \bf{v}_j) \\ \hline 
    3 & 4 & -64 & 36 & 20\\ 3 & 5 & -32 & -12 & 20\\ 4 & 1 & 16 & -8 & -4\\ 4 & 2 & 16 & 8 & -28\\ 4 & 3 & 32 & -4 & -28\\ 4 & 5 & -16 & -4 & 28\\ 5 & 1 & -16 & 16 & -16\\ 5 & 2 & 48 & -16 & 16\\ 5 & 3 & 32 & -16 & 16\\ 5 & 4 & -32 & 48 & -16 \end{array}\right)
\]

Now consider the following modification of the above example 
obtained by perturbing $\bf{v}_5$:
\begin{align} \label{eq:werner chiral}
U =             \left(\begin{array}{ccccc} 0 & 0 & 4 & 2 & 2 \\ 0 & 4 & 0 & 1 & 3 \\ 1 & 1 & 1 & 1 & 1 \end{array}\right)
 \,\,\,\textup{ and } \,\,\,
      V =         \left(\begin{array}{ccccc} 2 & 2 & 4 & 0 & 4 \\ 1 & 3 & 0 & 4 & 4 \\ 1 & 1 & 1 & 1 & 1 \end{array}\right).
\end{align}
We compute the products $D_{lm}, D_{ln}, D_{mn}$ at all 20 corners $(\bf{u}_i,\bf{v}_j)$ and find that
\[
D_{14}(\bf{u}_2, \bf{v}_3) = -32, \quad D_{15}(\bf{u}_2, \bf{v}_3) = -64, \quad D_{45}(\bf{u}_2, \bf{v}_3) = -64, 
\]
and hence the $(\bf{u}_2,\bf{v}_3)$ corner lies in the boundary of the chiral epipolar region, so these point pairs have a chiral reconstruction. We point out that the same-signness property is not symmetric in the sense that if it holds for 
$(\bf{u}_i,\bf{v}_j)$, it need not hold for $(\bf{u}_j,\bf{v}_i)$. 
Indeed, \[
D_{14}(\bf{u}_3, \bf{v}_2) = 16, \quad D_{15}(\bf{u}_3, \bf{v}_2) = -48, \quad D_{45}(\bf{u}_3, \bf{v}_2) = 16.
\]

\end{example}

The point pairs in \eqref{eq:werner no chiral} are a slight modification of \cite[Figure 1]{werner2003constraint} for which the epipolar variety is singular. Our modification makes the  variety smooth, a property of interest in the next section.
Furthermore, since sufficiently small perturbations of point pairs do not change the signs of $D_{lm}$, $D_{ln}$, and $D_{mn}$ at any $(\bf{u}_i, \bf{v}_j)$ corner, our methods show that having no chiral reconstruction is an open condition.
This leads to the following conclusion.

\begin{theorem} \label{thm:Zariski dense}
The set of five point pairs without a chiral reconstruction is Zariski dense in the space of all five point pairs.
\end{theorem}

Finally, embedding the point pairs in \eqref{eq:werner no chiral} into instances of six or more point pairs one can construct larger configurations with no chiral reconstruction.

\section{Connections to Classical Algebraic Geometry}
\label{sec:classicalAG}

In this section, we discuss the connection between five point pairs and the classical theory of real cubic surfaces in projective $3$-space. This point of view is yet another way to study the space of epipoles, which goes back to Klein and Segre \cite{segre}. General references for cubic surfaces are still mostly classical books like \cite{Henderson},\cite{HilbertCohnVossen}, or \cite{segre}. More modern accounts of classical facts about cubic surfaces can be found in \cite{DolgachevClassical} or \cite{ReidUAG}. Some of the history going back to Cayley and Salmon is discussed in \cite{DolgachevLuigiCremona}.

As in \Cref{sec:fivepoints}, we consider sets of five point pairs $\mathcal{P}$ that are generic though some of our discussion below might generalize to more singular situations.
In the first three subsections, we use the following as a running example.

\begin{example}\label{exm:7running}
    Consider the point correspondences $\mathcal{P} = \{(\bf{u}_i, \bf{v}_i), \,\,i=1,\ldots,5\}$ where $\bf{u}_i$ is the $i$th column of $U$ and $\bf{v}_i$ is the $i$th column of $V$ shown below:
    \[ U = \begin{pmatrix}0&0&1&1&2\\
          1&0&1&2&{-1}\\
          1&1&1&1&1\\
          \end{pmatrix}, \,\,\,\,\,\,\,\,\,
          V = \begin{pmatrix}3&5&{-1}&{-3}&1\\
          0&0&{-2}&{-2}&4\\
          1&1&1&1&1\\
          \end{pmatrix}.\]
\end{example}

\subsection{From two images to a cubic surface}
Given five generic point pairs $\PP = \{ (\vu_i,\vv_i)\in \P_\R^2\times \P_\R^2\mid i=1,2,3,4,5\}$, the epipolar variety $\RR_2\cap \LL_\PP \subset \LL_\PP \simeq \P_\R^3$ is a smooth cubic surface and its defining equation comes with a determinantal representation. Indeed, $\RR_2\subset \P_\R^8$ is defined by $\det(X) = 0$; to get the equation of the intersection with the linear space $\LL_\PP\subset \P_\R^8$, we compute a basis $(M_0,M_1,M_2,M_3)$ of $\LL_\PP$ and plug a general linear combination $z_0 M_0 + z_1 M_1 + z_2 M_2 + z_3 M_3$ of this basis into the equation $\det(X)$ to obtain 
\[
\RR_2 \cap \LL_\PP = \{(z_0:z_1:z_2:z_3)\in\P_\R^3 \mid \det(M(z_0,z_1,z_2,z_3)) = 0\},
\]
where the entries of $M$ are linear forms in $z_0,z_1,z_2,z_3$.
\begin{example}
    The variety $\mathcal{R}_2 \cap \mathcal{L}_\mathcal{P}$ is cut out by 
    $\det M(\mathbf{z}) = 0$ where 
    \[ M(\bf{z}) = \begin{pmatrix}-{z}_{0}-{z}_{1}+{z}_{2}+4 {z}_{3}&
          -{z}_{2}-{z}_{3}&
          -{z}_{2}+{z}_{3}\\
          {z}_{0}-{z}_{1}&
          {z}_{0}+{z}_{1}+3 {z}_{2}&
          -{z}_{0}+2 {z}_{1}+{z}_{2}-{z}_{3}\\
          {z}_{0}+3 {z}_{1}+{z}_{2}+2 {z}_{3}&
          {z}_{2}+5 {z}_{3}&
          5 {z}_{2}-5 {z}_{3}\\
          \end{pmatrix}.\]
          The coefficient matrices $M_0,M_1,M_2,M_3$ in the linear matrix polynomial $M(\bf{z}) = z_0 M_0 + z_1 M_1 + z_2 M_2 + z_3 M_3$ form a basis of $\mathcal{L}_\mathcal{P}$.
\end{example}

\subsection{The 27 lines on a cubic surface}
The cubic surface $S = \RR_2 \cap \LL_\PP\subset \P_\R^3$ contains special lines coming from the input point pairs, namely the walls $W_{\vu_i}$ and $W^{\vv_j}$. From \Cref{lem:k5corners}, we know the lines $W_{\vu_i}$ do not intersect each other and the lines $W^{\vv_j}$ do not intersect each other. Additionally the lines $W_{\vu_i}$ intersect the lines $W^{\vv_j}$ in corners, but they do not all intersect pairwise. We have $10$ lines in a very special configuration: It turns out that this is only possible if these $10$ lines are contained in a real \emph{Schl\"afli double six} on $S$.


\begin{definition}
A Schl\"afli double six on a smooth cubic surface $S \subset \P_\R^3$ is a pair of six-tuples of lines \[ \{(\ell_1,\ell_2,\ldots,\ell_6),(\ell'_1,\ell'_2,\ldots,\ell'_6)\} \] on $S$ such that the six lines in each tuple are pairwise disjoint and $\ell_i \cap \ell'_j$ is non-empty if and only if $i\neq j$. 
\end{definition}

Every smooth cubic surface contains $27$ complex lines in total and they can be organized into $36$ different Schl\"afli double six configurations. The combinatorics behind this has been studied extensively. A discussion of the Schl\"afli double sixes is summarized in \cite[Section 2]{buckley2007determinantal}. For a development via the blow-up construction, see \cite[Chapter V, Section 4]{hartshorne1977algebraic}. With this approach, it is straightforward to check that a pair of five tuples of lines $((\ell_1,\ldots,\ell_5),(\ell'_1,\ldots,\ell'_5))$ on a smooth cubic surface with $\ell_i \cap \ell_j = \emptyset$, $\ell_i'\cap \ell_j' = \emptyset$ (for all $i\neq j$), and $\ell_i \cap \ell_j' = \emptyset$ if and only if $i = j$, uniquely determines a Schl\"afli double six. 

The fact that our cubic surfaces, that appear as epipolar varieties, always contain a Schl\"afli double six consisting of real lines, means that they are all of the same real topological type. Schl\"afli, Klein, and Zeuthen classified the real topologies of cubic surfaces. There are five possible types. The real classification of cubic surfaces is summarized in \cite[Theorem 5.4]{buckley2007determinantal} and discussed in detail in \cite[pp.~40--55]{Geramita}. Epipolar varieties are always of type $F_1$ containing $27$ real lines. Indeed, cubic surfaces of type $F_3$, $F_4$, or $F_5$ do not contain enough real lines to have a real Schl\"afli double six (namely only seven, three, and three). Type $F_2$ surfaces contain $15$ real lines, which would be enough for a Schl\"afli double six. These surfaces are the blow-up of $\P_\R^2$ in four real points and one conjugate pair of complex points. The $15$ lines are the four real exceptional divisors over the four real points, the strict transforms of the six real lines joining any pair of the four real points, the strict transform of the one real line joining the conjugate pair of complex points and the strict transforms of the four conics passing through all sets of five points that is complementary to one of the real points. It is straightforward to check that there are no six mutually skew lines among these $15$ real lines.

Every smooth cubic surface that arises as the epipolar variety of five point pairs is the blow-up of the real projective plane $\P_\R^2$ in \emph{six} real points in general position. (The other types are obtained by blowing up $\P_\R^2$ in six complex points that are invariant under complex conjugation with different numbers of fixed points, namely $0$, $2$, and $4$. The last remaining type is not isomorphic to a blow-up of $\P_\R^2$ over the reals.) This is curious because we start with five point pairs that specify $10$ lines on the epipolar variety, but these $10$ lines determine two other lines via the unique Schl\"afli double six containing the $10$. This prescribes one more point in every image, as we will see below. In \cite{werner2003constraint}, Werner shows one way to construct this sixth point in each image using the five points pairs.
We discuss two other ways to explain the sixth point (see \Cref{exm:sixth1} and \Cref{exm:sixth2}).

\begin{example}\label{exm:sixth1}
    The four-dimensional linear space $\Span\{\bf{v}_1\bf{u}_1^\top, \bf{v}_2\bf{u}_2^\top, \dots, \bf{v}_5\bf{u}_5^\top\}$ intersects the four-dimensional, degree six variety $\RR_1$ in six points. The sixth point is $\bf{v}_0\bf{u}_0^\top$ where $\bf{v}_0 = (-3,-12,5)$ and $\bf{u}_0 = (18,11,17)^\top$.
\end{example}

The Schl\"afli double six specified by the lines $W_{\bf{u}_i}$ and $W^{\bf{v}_j}$ for $0\le i,j\le 5$ accounts for twelve of the 27 lines on the epipolar variety. For $i\neq j$, $ W_{\bf{u}_i}$ and $W^{\bf{v}_j}$ span a tritangent plane, $\pi_i^j$. The tritangent plane intersects the epipolar variety in $W_{\bf{u}_i}$, $W^{\bf{v}_j}$ and one additional line 
\begin{align*}
 W_i^j := \{ X \in \RR_2 \cap L_\mathcal{P} : 
 X\bf{u} = 0 \textup{ and } \bf{v}^\top X = 0 \textup{ for some } \bf{u} \in \Span\{\bf{u}_i, \bf{u}_j\} , \bf{v} \in \Span\{\bf{v}_i, \bf{v}_j\} \}.
 \end{align*}
The plane $\pi_i^j$ coincides with $\pi_j^i$, hence the line $W_i^j$ coincides with $W_j^i$ for all $i,j$. The distinct lines $W_i^j$ for $0\le i<j\le 5$ account for the remaining 15 real lines on the epipolar variety.

\subsection{The determinantal representations of a cubic surface}
From the classical point of view, we can start with six real points in the real projective plane $\P_\R^2$ and obtain a cubic surface of the correct topological type. Each of these surfaces is the epipolar variety for five generic points pairs but every cubic surface is compatible with many second images. To determine the second image, we go via a determinantal representation; different choices of a determinantal representation result in different second images (even modulo projective transformations).

Given six real points $\vu_0,\ldots,\vu_5$ in $\P_\R^2$, we get a cubic surface $S\subset \P_\R^3$ that is the blow-up of $\P_\R^2$ in these six points by considering the rational map $\P_\R^2\dashrightarrow \P_\R^3$ given by the four-dimensional space of cubics that vanish at the six points. Fixing a basis $c_0,c_1,c_2,c_3$ of this space, the map is given by $\vx \mapsto (c_0(\vx):c_1(\vx):c_2(\vx):c_3(\vx))$ and is defined on $\P_\R^2\setminus\{\vu_0,\ldots,\vu_5\}$. The closure of the image of this map is the cubic surface $S$ in $\P_\R^3$ (which is determined up to change of coordinates on $\P_\R^3$ by the six points in $\P_\R^2$).

We cannot determine the second image from this surface alone. However, this information is specified by a determinantal representation of the surface. Determinantal representations are closely related to the \emph{Hilbert-Burch matrix} as is explained in \cite[Section 3]{buckley2007determinantal} (see also \cite{Geramita}). The vanishing ideal $I$ of six points in $\P_\R^2$ in general linear position is generated by the four-dimensional linear space of cubics vanishing on it. Its minimal free resolution looks like
\[
    0 \to F \to G \to I \to 0
\]
where $F$ is a free, graded $\R[x_0,x_1,x_2]$-module of rank three and the map from $F$ to $G$ is given by a $4\times 3$ Hilbert-Burch matrix $L(x_0,x_1,x_2)^\top$ with entries linear in $x_0,x_1,x_2$. Every such matrix gives a determinantal representation of $S$ via a simple computation: Determine the $3\times 3$ matrix $M$ with entries in $\R[z_0,z_1,z_2,z_3]_1$ such that 
\[
   L(x_0,x_1,x_2)
   \begin{pmatrix}
       z_0 \\ z_1 \\ z_2 \\ z_3
   \end{pmatrix} =
   M(z_0,z_1,z_2,z_3)
   \begin{pmatrix}
       x_0 \\ x_1 \\ x_2
   \end{pmatrix} 
\]
and then $S = \{(z_0,z_1,z_2,z_3)\in \P_\R^3 \mid \det(M(\bf{z})) = 0\}$. Since $M(\bf{z})
^\top$ is also a determinantal representation of $S$, we get another Hilbert-Burch matrix $L'(x_0,x_1,x_2)^\top$ with linear entries such that 
 \[
   L'(x_0,x_1,x_2)
   \begin{pmatrix}
       z_0 \\ z_1 \\ z_2 \\ z_3
   \end{pmatrix} =
   \left[M(z_0,z_1,z_2,z_3)\right]^\top
   \begin{pmatrix}
       x_0 \\ x_1 \\ x_2
   \end{pmatrix}.
\]

Therefore a determinantal representation determines six real points $\vu_i$ in the first image cut out by the $3\times 3$ minors of $L$ and six real points $\vv_i$ in the second image cut out by the $3\times 3$ minors of $L'$.

\begin{example}\label{exm:sixth2}
    The Hilbert Burch matrix of $M(\mathbf{z})$ is
    \[
    L = \begin{pmatrix}
          -x_0 & -x_0 & x_0 - x_1 - x_2 & 4\, x_0 - x_1 + x_2 \\
          x_0 + x_1 - x_2 & -x_0 + x_1 + 2\,x_2 & 3\,x_1 + x_2 & -x_2 \\
          x_0 & 3\,x_0 & x_0 + x_1 + 5\,x_2 & 2\,x_0 + 5\,x_1 - 5\,x_2 \end{pmatrix}
    \]
    while the Hilbert-Burch matrix of $M(\bf{z})^\top$ is
    \[
      L' = \begin{pmatrix}
      -{x}_{0}+{x}_{1}+{x}_{2}&-{x}_{0}-{x}_{1}+3\,{x}_{2}&{x}_{0}+{x}_{2}&4\,{x}_{0}+2\,{x}_{2}\\
      {x}_{1}&{x}_{1}&-{x}_{0}+3\,{x}_{1}+{x}_{2}&-{x}_{0}+5\,{x}_{2}\\
      {-{x}_{1}}&2\,{x}_{1}&-{x}_{0}+{x}_{1}+5\,{x}_{2}&{x}_{0}-{x}_{1}-5\,{x}_{2}\end{pmatrix}.
    \]
    We can compute the six real points in each image from the Hilbert-Burch matrices, because the $3\times 3$ minors vanish exactly at these points. The zeros of these ideals of minors are 
    $\mathbf{u}_0, \ldots, \mathbf{u}_5$ where $\bf{u}_0 = (18,11,17)^\top$ and $\mathbf{v}_0, \ldots, \mathbf{v}_5$ where $\bf{v}_0 = (-3,-12,5)^\top$, respectively. This gives yet another way to compute $\bf{u}_0$ and $\bf{v}_0$.
\end{example}

\subsection{From a determinantal representation to epipoles}
Given a determinantal representation $M(\bf{z})$ of a cubic surface $S \subset \P_\R^3$, consider the map $S \to \P_\R^2\times \P_\R^2$, $\bf{z} \mapsto \adj(M(\bf{z}))$. Since the matrix $M(\bf{z})$ has rank two for every point $\bf{z}\in S$, the image lies indeed inside $\P_\R^2\times \P_\R^2$ embedded via the standard Segre embedding in $\P_\R^8$. With its given determinantal representation, the image of the cubic surface $S$ under the adjoint map is the space of epipoles in two images consistent with $\PP$.

From our study of the chiral epipolar region in \Cref{sec:fivepoints}, we are interested in the lines $W_{\vu_i}$ and $W^{\vv_j}$ on the cubic surface corresponding to $\PP$. We compute the images of these lines under the adjoint map, connecting our work with Werner's results in epipole space \cite{werner2003constraint}. The image of $W_{\vu_i}$ under the adjoint map is $\{\vu_i\}\times C^i$, where $C^i\subset \P_\R^2$ is a curve of degree two, namely the conic passing through $\vv_j$ for $i\neq j$. Indeed, for every matrix $X\in W_{\vu_i}$, the vector $\vu_i$ is in the right kernel, by definition, whereas the left kernel varies. This left kernel varies along a conic in $\P_\R^2$ because every entry of the adjoint matrix of a $3\times 3$ matrix is a quadric. So intersecting the image of $W_{\vu_i}$ under the adjoint map with a line in $\P_\R^2$ translates, by pulling back via the adjoint map, to the intersection of a quadric hypersurface in $\P_\R^3$ with the line $W_{\vu_i}$, which therefore consists of two intersection points. Symmetrically, the image of $W^{\vv_j}$ is $C_j\times \{\vv_j\}$, where $C_j\subset \P_\R^2$ is the conic passing through $\vu_j$ for $i\neq j$. This behaviour is special and occurs only for the walls in the cubic surface. The image of the other lines $W_i^j$ under the adjoint is contained in $\Span\{\bf{u}_i, \bf{u}_j\} \times \Span\{\bf{v}_i, \bf{v}_j\}$. 



The conics $C_i$ in the first image and $C^j$ in the second image are described in Section 4 of \cite{werner2003constraint}. These conics divide each image into cells of possible epipoles with piecewise conic boundaries. Each cell is uniquely characterized by the subset of conics that participate to form its boundary or equivalently by the subset of points that belong to its boundary. Werner's test for the existence of a chiral reconstruction involves looking for ``allowed segments" of the conics $C_i$ and $C^j$. In our work, we translate this question to studying the intersections of lines, i.e., the $(\bf{u}_i, \bf{v}_j)$ corners in the cubic surface in the preimage of the adjoint. In the following example, we show how our chirality test on corners relates to Werner's allowed segments in the space of epipoles.

\begin{example*}[\textbf{\ref{ex:Werner} continued}] 
Consider the five point pairs from \eqref{eq:werner chiral}:
\begin{align*}
U =             \left(\begin{array}{ccccc} 0 & 0 & 4 & 2 & 2 \\ 0 & 4 & 0 & 1 & 3 \\ 1 & 1 & 1 & 1 & 1 \end{array}\right)
 \,\,\,\textup{ and } \,\,\,
      V =         \left(\begin{array}{ccccc} 2 & 2 & 4 & 0 & 4 \\ 1 & 3 & 0 & 4 & 4 \\ 1 & 1 & 1 & 1 & 1 \end{array}\right).
\end{align*}
We find that the same sign condition of \Cref{cor:5 point det at corner} holds at the following corners: $(\bf{u}_2, \bf{v}_3)$, $(\bf{u}_2, \bf{v}_4)$, $(\bf{u}_3, \bf{v}_1)$, $(\bf{u}_4, \bf{v}_1)$, and $(\bf{u}_4, \bf{v}_3)$. These corners are in the boundary of the chiral epipolar region of $\mathcal{P}$ which lives in the cubic surface $\RR_2 \cap \mathcal{L}_\mathcal{P}$. Blowing down with the adjoint map, we get the non-empty region of epipoles in $\P_\R^2 \times \P_\R^2$ that correspond to a chiral reconstruction. The boundary of the region in the first image is defined by $C_3$ and $C_4$ at $\bf{u}_2$, $C_1$ and $C_4$ at $\bf{v}_3$, and $C_1$ and $C_3$ at $\bf{u}_4$. The boundary of the region in the second image is defined by $C^3$ and $C^4$ at $\bf{v}_1$, $C^2$ and $C^4$ at $\bf{v}_3$, and $C^2$ and $C^3$ at $\bf{v}_4$. See \Cref{fig:shaded chiral region}.
\end{example*}

\begin{figure}
\begin{subfigure}{.49\textwidth}
  \centering
  \includegraphics[width=\linewidth]{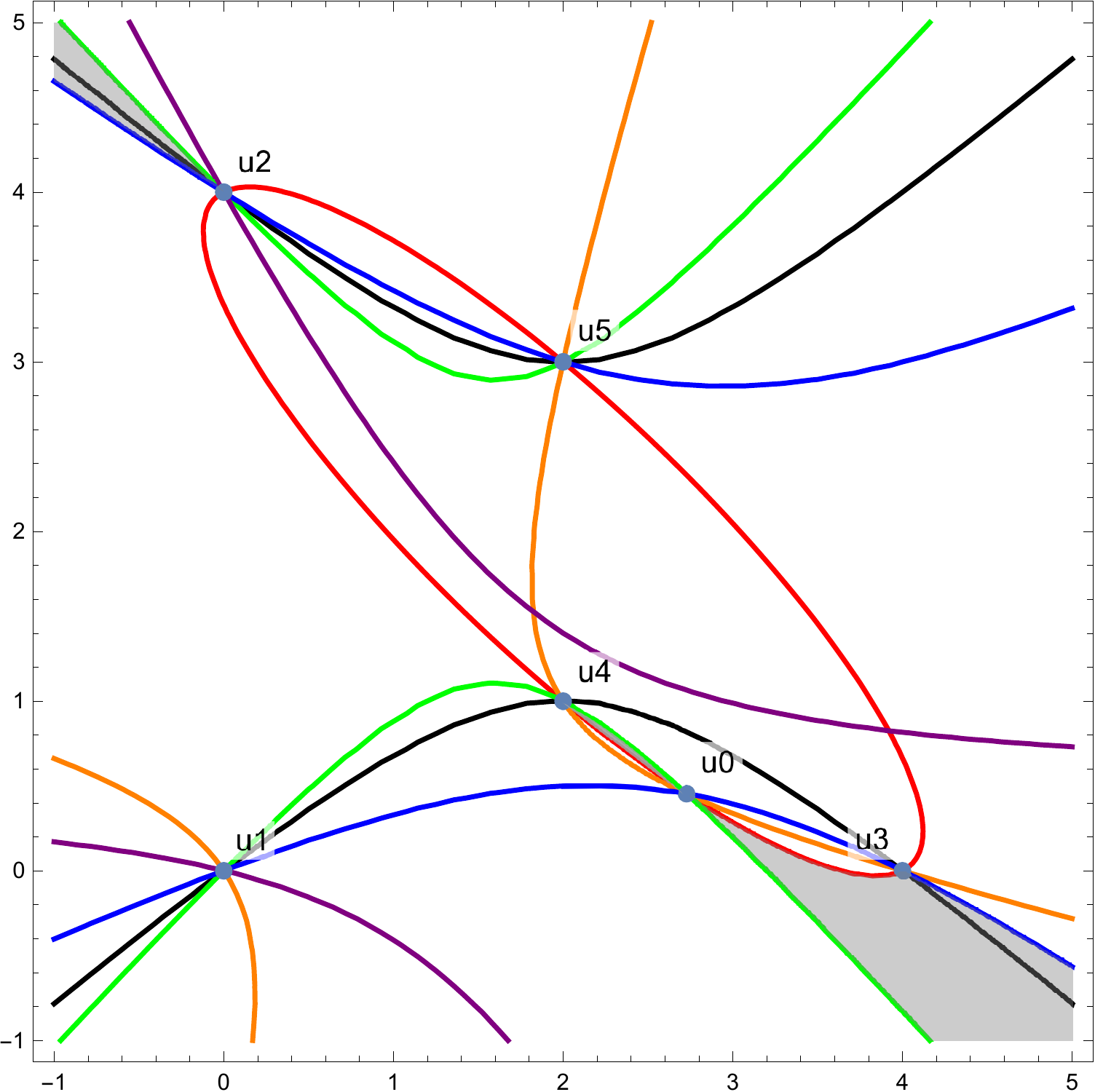} 
  \caption{Shaded area represents region of allowed epipoles $\bf{e}_1$. }
\end{subfigure}
\begin{subfigure}{.49\textwidth}
  \centering
  \includegraphics[width=\linewidth]{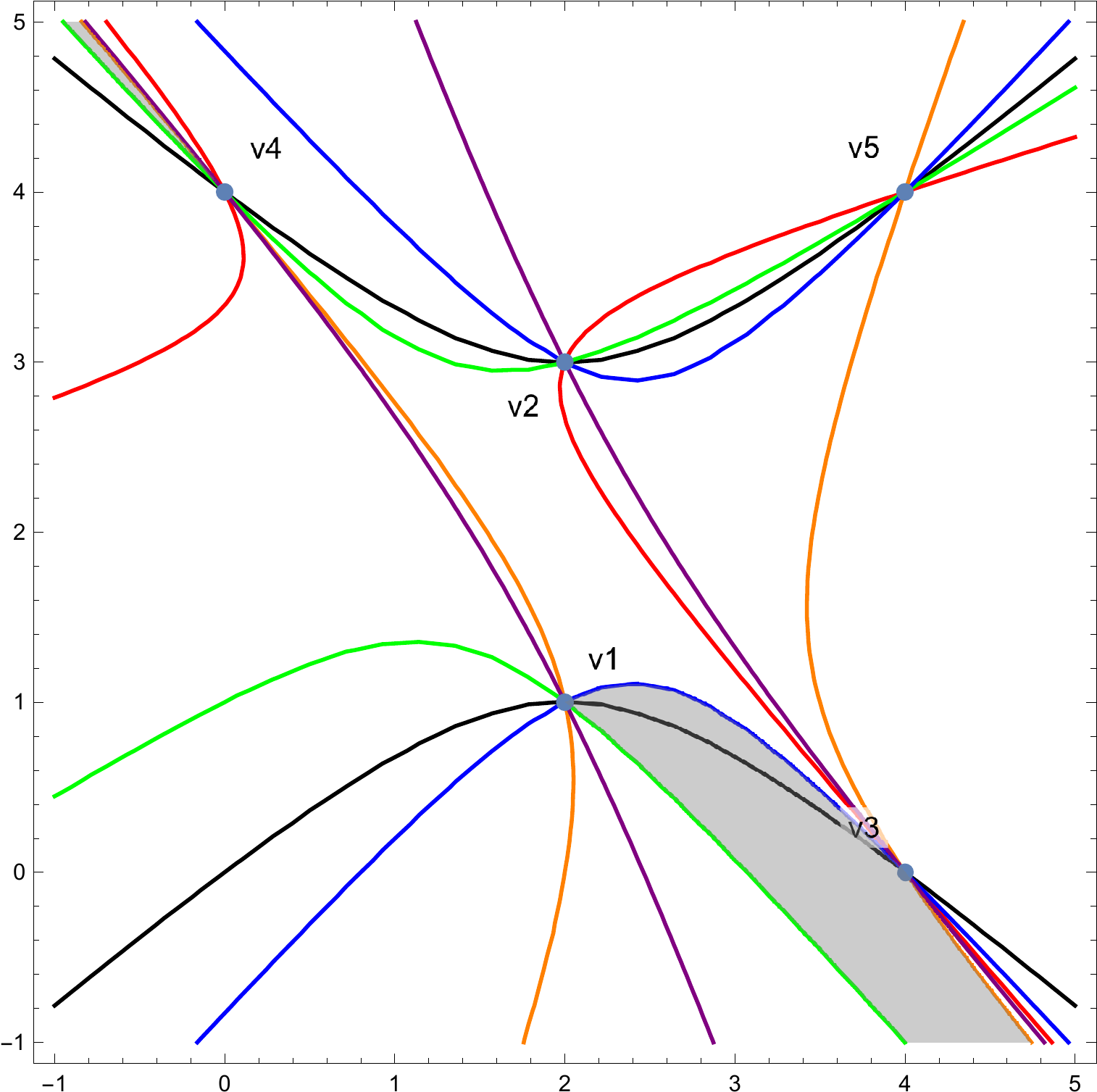} 
  \caption{Shaded area represents region of allowed epipoles $\bf{e}_2$.}
\end{subfigure}
\caption{}
\label{fig:shaded chiral region}
\end{figure}

We remark that our test for chirality requires more than just the cubic surface $S$ determined by six real points $\vu_i$ in one image $\P_\R^2$. For different determinantal representations of $S$, the six points in the other image may differ. In order to check the chiral epipolar inequalities, it is necessary to fix the second image, for example, by fixing a determinantal representation of $S$.

It is also necessary to specify five out of the six point pairs to test for chirality. As the next example shows, while some sets of five point pairs may have a chiral reconstruction, another set associated with the same determinantal representation may not. These observations point out the intricate dependence of chirality on arithmetic information. 

\begin{example*}[\textbf{\ref{ex:Werner} continued}] 
Consider the five point pairs from \eqref{eq:werner no chiral}:
\begin{align*}
U =             \left(\begin{array}{ccccc} 0 & 0 & 4 & 2 & 2\\ 0 & 4 & 0 & 1 & 3\\ 1 & 1 & 1 & 1 & 1 \end{array}\right)
 \,\,\,\,\,\,\,\,\,
      V =         \left(\begin{array}{ccccc} 2 & 2 & 4 & 0 & 1\\ 1 & 3 & 0 & 4 & 1 \\ 1 & 1 & 1 & 1 & 1 \end{array}\right).
\end{align*}
We compute the sixth associated point pair $\bf{u}_0 = (504/281,300/281,1)^\top$ and $\bf{v}_0 = (68/97, 300/97, 1)^\top$. 
Note that the need for all point pairs to have last coordinate one fixes the representatives of $\mathbf{u}_0$ and $\mathbf{v}_0$. By computing the relevant products $D_{lm}, D_{ln}, D_{mn}$ at the 20 corners, we find that $\{(\bf{u}_i, \bf{v}_i)\}_{i=1}^5$ does not have a chiral reconstruction. However, suppose we consider a different subset of five point pairs
\[
\widehat{\mathcal{P}}_5 = \{(\bf{u}_i,\bf{v}_i) : i \in \{0,1,2,3,4\}\}.
\]
Some corners now satisfy the same sign condition of \Cref{cor:5 point det at corner}, hence $\widehat{\mathcal{P}}_5$ has a chiral reconstruction. In fact, the subset of five point pairs $\widehat{\mathcal{P}}_i$ which omits index $i$ has a chiral reconstruction when $i = 1,2,3,5$ and has no chiral reconstruction when $i = 0,4$. The cubic surface $\RR_2 \cap L_{\widehat{\mathcal{P}_i}}$ is the same for all $i$ and the determinantal representations are all equivalent, but whether a chiral reconstruction exists or not depends on which five point pairs we choose. 
\end{example*}



\bibliography{references}{}
\bibliographystyle{siam}

\end{document}